\documentclass[letterpaper]{article} % DO NOT CHANGE THIS
\usepackage{aaai23}  % DO NOT CHANGE THIS
\usepackage{times}  % DO NOT CHANGE THIS
\usepackage{helvet}  % DO NOT CHANGE THIS
\usepackage{courier}  % DO NOT CHANGE THIS
\usepackage[hyphens]{url}  % DO NOT CHANGE THIS
\usepackage{graphicx} % DO NOT CHANGE THIS
\usepackage{natbib}  % DO NOT CHANGE THIS AND DO NOT ADD ANY OPTIONS TO IT
\usepackage{caption} % DO NOT CHANGE THIS AND DO NOT ADD ANY OPTIONS TO IT
\usepackage{algorithm}
\usepackage{algorithmic}

\usepackage{newfloat}
\usepackage{listings}
\DeclareCaptionStyle{ruled}{labelfont=normalfont,labelsep=colon,strut=off} % DO NOT CHANGE THIS
\floatstyle{ruled}
\newfloat{listing}{tb}{lst}{}
\floatname{listing}{Listing}
\usepackage[utf8]{inputenc} % allow utf-8 input
\usepackage[T1]{fontenc}    % use 8-bit T1 fonts
\usepackage{url}            % simple URL typesetting
\usepackage{booktabs}       % professional-quality tables
\usepackage{amsfonts}       % blackboard math symbols
\usepackage{nicefrac}       % compact symbols for 1/2, etc.
\usepackage{microtype}      % microtypography
\usepackage{xcolor}         % colors
\usepackage{graphicx}
\usepackage{comment}
\usepackage{amsmath,amssymb} % define this before the line numbering.
\usepackage{color}
\usepackage{microtype}
\usepackage{graphicx}
\usepackage{subfigure}
\usepackage{booktabs} % for professional tables
\usepackage{multirow}
\usepackage{diagbox}
\usepackage{makecell}
\usepackage{amsthm}
\usepackage{mathtools}
\usepackage{pifont}% http://ctan.org/pkg/pifont
\usepackage{siunitx,booktabs,etoolbox}
\usepackage{kotex}

\newcommand{\eat}[1]{}

\newcommand{\smalltitle}[1]{ \vspace{1mm}{\noindent\textbf{#1.}\hspace{1mm}}}

\newcommand{\cmark}{\ding{51}}%
\newcommand{\xmark}{\ding{55}}%

\newtheorem{proposition}{Proposition}

\def\D{\mathcal{D}}

\def\L{\mathcal{L}}

\def\R{\mathcal{R}}
\def\naive{na\"ive }

\DeclareMathOperator*{\argmin}{arg\,min}

\DeclarePairedDelimiter\abs{\lvert}{\rvert}%

\title{Better Generalized Few-Shot Learning Even Without Base Data}
\author{
    Seong-Woong Kim, Dong-Wan Choi\thanks{Corresponding Author}
}
\affiliations{
    Department of Computer Science and Engineering, Inha University, South Korea\\
    wauri6@gmail.com, dchoi@inha.ac.kr
}

\usepackage{bibentry}
\begin{document}

\maketitle

\begin{abstract}
This paper introduces and studies \textit{zero-base generalized few-shot learning} (\textit{zero-base GFSL}), which is an extreme yet practical version of few-shot learning problem. Motivated by the cases where base data is not available due to privacy or ethical issues, the goal of zero-base GFSL is to newly incorporate the knowledge of few samples of novel classes into a pretrained model without any samples of base classes. According to our analysis, we discover the fact that both mean and variance of the weight distribution of novel classes are not properly established, compared to those of base classes. The existing GFSL methods attempt to make the weight norms balanced, which we find helps only the variance part, but discard the importance of mean of weights particularly for novel classes, leading to the limited performance in the GFSL problem even with base data. In this paper, we overcome this limitation by proposing a simple yet effective normalization method that can effectively control both mean and variance of the weight distribution of novel classes without using any base samples and thereby achieve a satisfactory performance on both novel and base classes. Our experimental results somewhat surprisingly show that the proposed zero-base GFSL method that does not utilize any base samples even outperforms the existing GFSL methods that make the best use of base data. Our implementation is available at: \url{https://github.com/bigdata-inha/Zero-Base-GFSL}.
\end{abstract}

%   Few-shot learning is now a major problem in deep learning communities motivated by the practical difficulty of data collection and the ability of human to acquire general knowledge with only a few observations.

% The proposed method is designed by our own findings about what internally occurs to the weight distribution of the model being trained during few-shot learning.
\section{Introduction} \label{sec:intro}

% 1.Few-shot learning has extensively studied.
Few-shot learning (FSL) \cite{Fei-FeiFP06,LakeSGT11,lake15} is now a major problem motivated by the practical difficulty of data collection. Trying to mimic the human's ability to acquire general knowledge with a few observations, the goal of FSL is to learn new knowledge by training only a few samples on a model, which is often a pretrained model. Many existing works focus on a scenario where the resulting model only discriminates novel classes, referred to as \textit{standard few-shot learning} \cite{FinnAL17,koch15, SnellSZ17,VinyalsBLKW16}. It is more challenging yet more practical when the model can infer both existing classes (a.k.a. base classes) already trained in the model and unseen classes (a.k.a. novel classes) newly learned with a few samples. Only some recent works \cite{GidarisK18,GidarisK19,KuklevaKS21,ShiSSW20} tackle this version of FSL, namely \textit{generalized few-shot learning} (\textit{GFSL}), which not only aims to learn about novel classes but also preserve the existing knowledge for base classes.

%2. Existing works on GFSL base transfer learning approach.
Due to its simplicity and performance, transfer learning becomes a dominant approach in GFSL, where we fine-tune a model already trained over base classes to additionally learn the knowledge of new classes with a few samples. Despite the small volume of new data, the fine-tuned model gets easily biased to novel classes to the point that the model becomes pretty useless for base classes. Hence, most existing works in GFSL focus on how to make the resulting model well balanced over base and novel classes by performing balanced fine-tuning \cite{KuklevaKS21,LiLX0019,QiBL18} or leveraging additional architectures \cite{GidarisK18,RenLFZ19,YoonKSM20} and supplement information \cite{Li0LFLW20,ShiSSW20}.

% figure 1. Decision boundaries in the feature
% \begin{figure}[t!]
% 	\centering
%     \includegraphics[width=0.95\columnwidth]{./figure/figure_features.pdf}
%     \caption{Conceptual visualization of the decision boundaries in the feature space of the model trained on base classes and then fine-tuned on novel classes, respectively, where red pentagons are feature vectors of novel classes and the other points are those of base classes. \textit{Left}: Well-formed decision boundaries of the pretrained model of base classes. \textit{Middle}: Fine-tuned decision boundaries biased toward novel classes in the same feature space. \textit{Right}: Ideal decision boundaries for both base and novel classes.}
%     \label{fig:boundaries}%Decision boundaries in the feature space
% 	%\vspace{-2mm}
% % 	Each colored area represents the region in which any feature vector is classified as the corresponding class. Lines are decision boundaries constructed by the linear classifier of the model.
% \end{figure}

\begin{figure}[t]
    \centering
    \includegraphics[height=5.8mm]{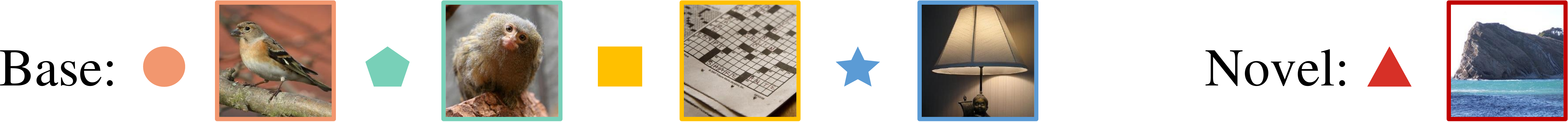} \\
    \subfigure[\label{fig:decision:a}Pretrained]{\hspace{1mm}\includegraphics[width=0.3\columnwidth]{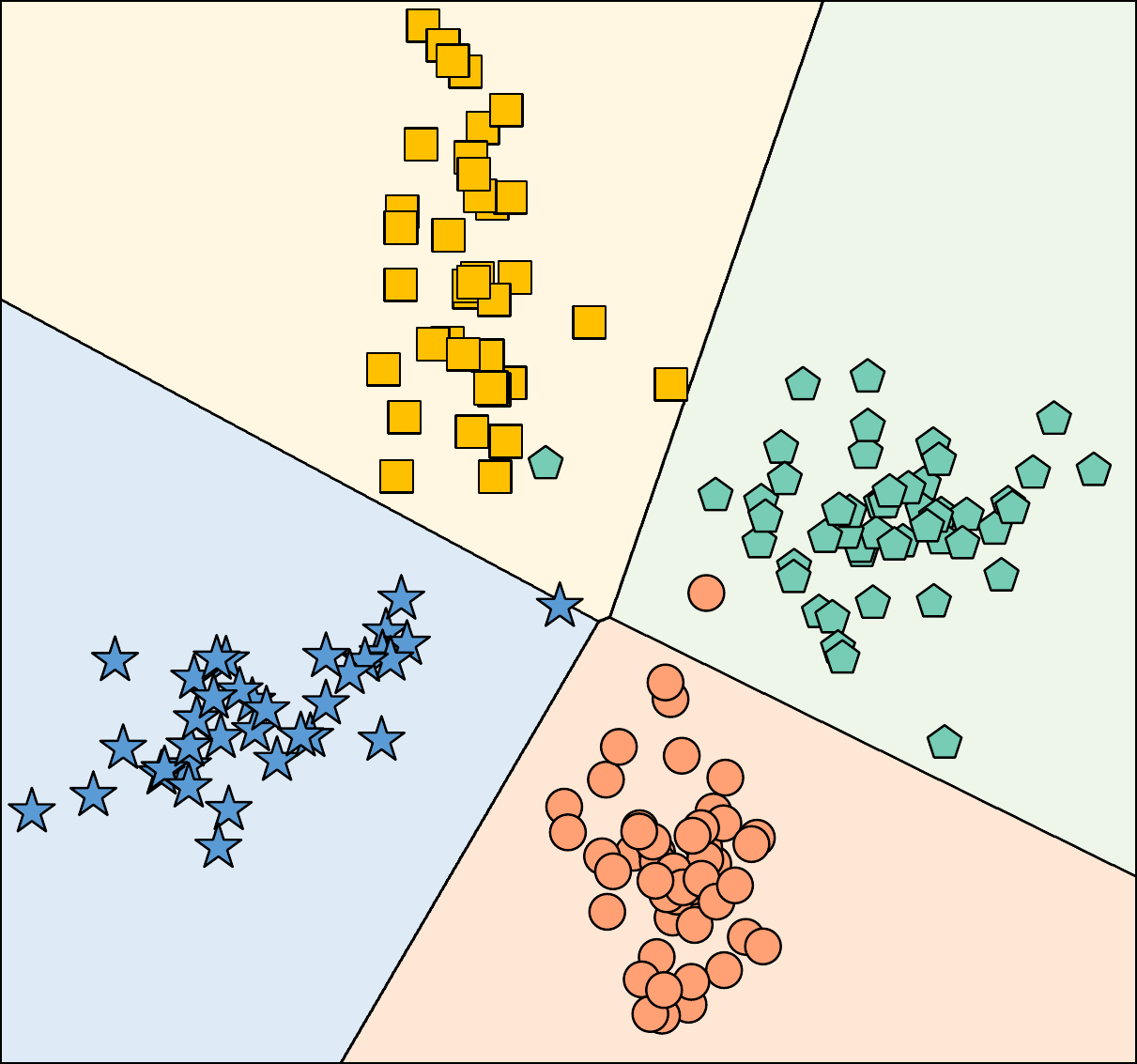}\hspace{1mm}}
    \subfigure[\label{fig:decision:b}Biased]{\hspace{1mm}\includegraphics[width=0.3\columnwidth]{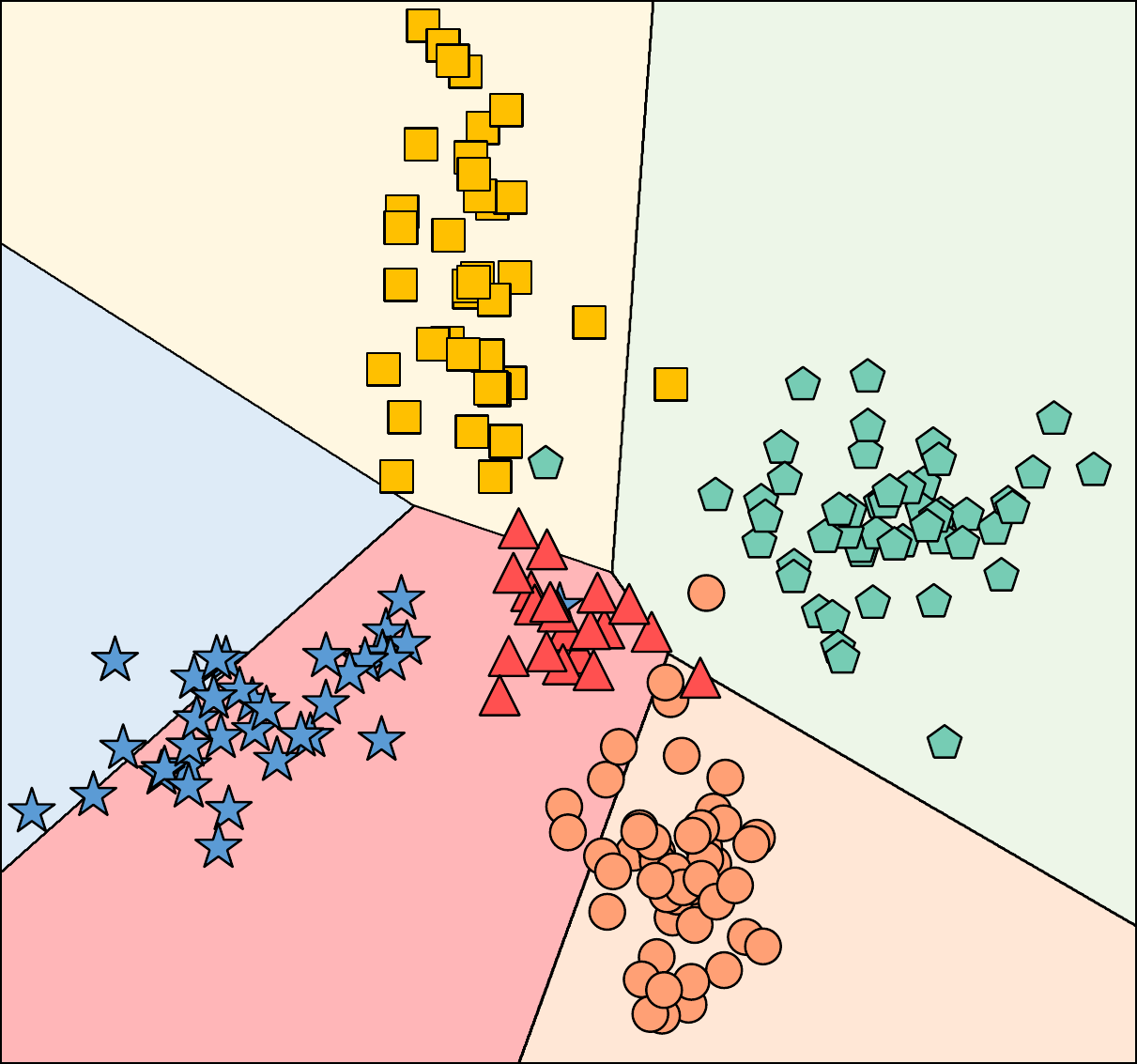}\hspace{1mm}}
    \subfigure[\label{fig:decision:c}Desired (ours)]{\hspace{1mm}\includegraphics[width=0.3\columnwidth]{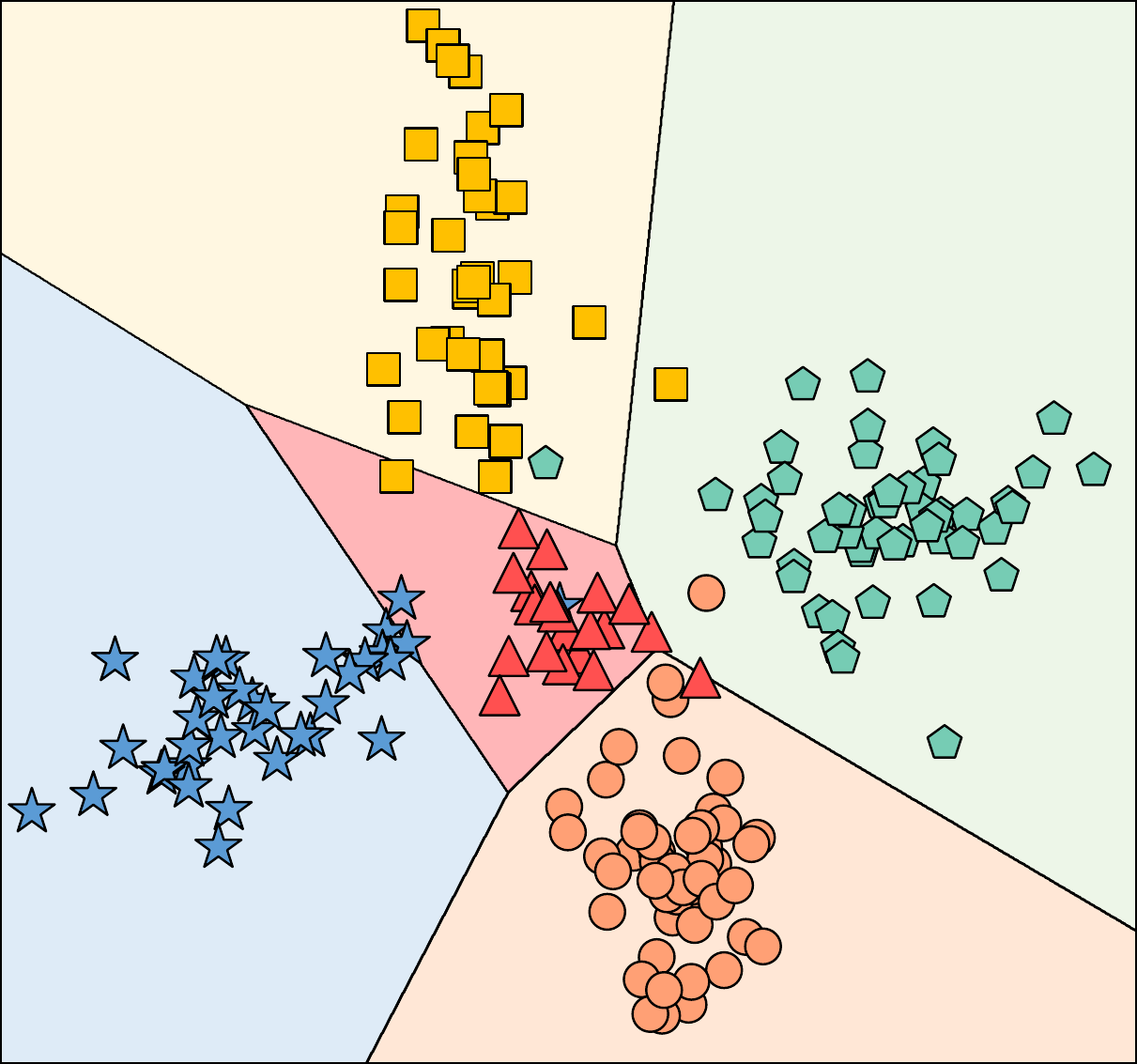}\hspace{1mm}}
    \caption{Visualization of the decision boundaries of ResNet-18 trained with an auxiliary 2D projection module before the classifier layer on \textit{tiered-ImageNet}, where red triangles are feature vectors of novel classes and the others are those of base classes. Note that feature vectors themselves remain the same in the frozen feature space. \textit{(a)}: Well-formed decision boundaries of the pretrained model of base classes. \textit{(b)}: Fine-tuned decision boundaries biased toward novel classes. \textit{(c)}: Well-balanced decision boundaries for both base and novel classes via our normalization.}
    %\textit{(a)}: Well-formed decision boundaries of the pretrained model of base classes.
	\label{fig:decision}
	\vspace{-3mm}
\end{figure}

% 3.However all requires base data itself or new architecture built upon base data
% No works on using only pretrained models, which we call zero-base GFSL
All the aforementioned approaches assume the presence of \textit{base data} (i.e., samples of base classes) for preserving the knowledge of base classes as well as learning the relationship between novel and base classes. In reality, however, base data may not be always available due to some privacy or ethical issues. For instance, \textit{Google} releases the highly generic model \textit{BiT} \cite{KolesnikovBZPYG20} of 18,000 classes, but its training data \textit{JFT} \cite{sun17} consists of millions of private images that should not be exposed to public. Furthermore, retraining the base data has never been a perfect solution for GFSL either. In spite of retraining overhead, the overall performance can be highly dependent on which base samples are selected. To be shown by our experimental results, the state-of-the-art GFSL methods using base data turn out to be even less accurate than our GFSL method without employing any base samples.

Beyond the limitation of the existing GFSL approaches, this paper focuses on an even more challenging scenario of few-shot learning, called \textit{zero-base GFSL}, where we are free to use a pretrained model somehow learned over base classes but cannot retrain any of base samples during GFSL. Obviously, fine-tuning a pretrained model with only novel samples leads to the model highly biased toward novel classes even if we freeze the feature extractor of the model. Thus, even in the frozen feature space built on base classes, novel samples are capable of forming undesirably large decision boundaries, which is hard to be addressed without jointly training all of the base and novel samples. Figure \ref{fig:decision:b} shows such an example where some base samples, which used to be well discriminated in their pretrained model (Figure \ref{fig:decision:a}), can be misclassified as a novel class with a large decision boundary. This leads to the model pretty inaccurate for base classes without normalization as shown in Figure \ref{fig:imagenet}.

Then, how could only a few novel samples make the decision boundaries highly biased toward novel classes in the feature space? To answer this question, we investigate the distribution of weights of classifiers that actually determine decision boundaries, and discover the following undesirable facts in terms of its mean and variance. First, there is an imbalance between the variances of weight distributions of base and novel classes, which we also find is related to the norm imbalance tackled by many existing GFSL methods \cite{FanMLS21,GidarisK18,QiBL18,WangH0DY20}. However, these methods do not consider how the means of weight distributions are different between base and novel classes. This paper newly observes that the average weight of the novel classifier is positively shifted from that of the base classifier, which we call \textit{mean shifting phenomenon}. Although this phenomenon is indeed a more crucial reason behind the biased model, existing works rarely try to fix the shifted mean of novel classes.

To remedy these undesirable mean and variance of novel classes, this paper designs a new normalization method that can achieve a centered mean as well as a balanced variance without retraining any base samples. As experimentally shown in Figure \ref{fig:decision:c}, our normalization method enables the classifier to form well-balanced decision boundaries for both base and novel classes, and thereby it is observed in Figure \ref{fig:imagenet} that the performance degradation of base classes is prevented while keeping the performance of the novel classes as high as possible. With only 5-shot of novel classes, we can keep the performance of both base and novel classes close to their upper bounds that are conditional accuracies given only base and novel classes, respectively. In our experiments, without using any base data, our method even beats the existing state-of-the-art GFSL methods with clear margins, like 4.59\% and 3.53\% 5-shot accuracy in \textit{mini-} and \textit{tiered-ImageNet}, respectively.

% figure 2. imagenet-800 1,5,10-shot linear vs normalization
\begin{figure}[t!]
	\centering
    \includegraphics[width=0.75\columnwidth]{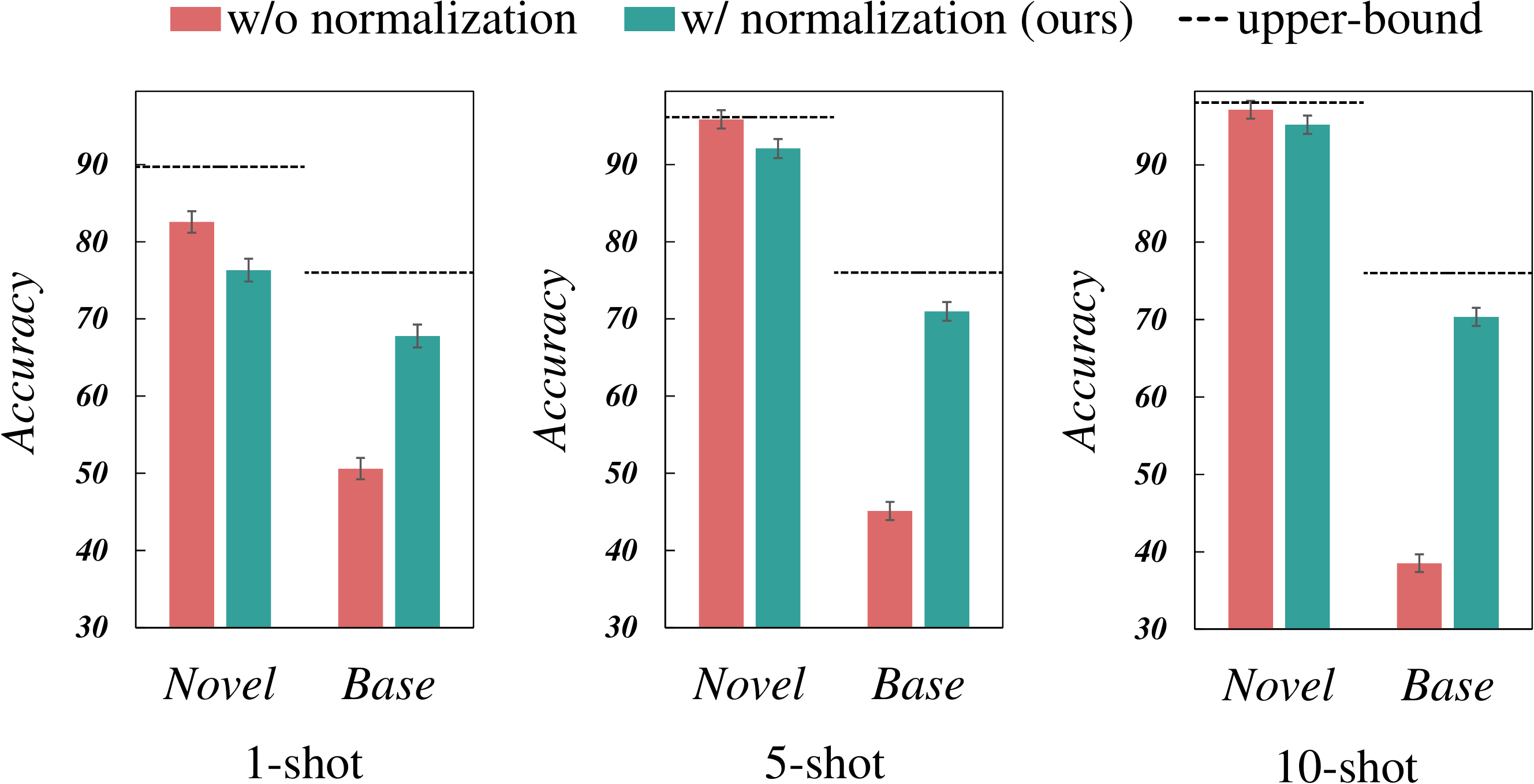}
    \caption{Comparison on the accuracy of zero-base GFSL with ResNet-50 on \textit{ImageNet-800}. The upper bound is the conditional test accuracy given either novel or base classes.}
	\label{fig:imagenet}%ImageNet-800 accuracy
	%\vspace{-2mm}
\end{figure}

\section{Related Works} \label{sec:related}

\smalltitle{Few-shot learning}
Few-shot learning (FSL) has been studied mostly for the standard scenario aiming to learn novel classes without having to preserve the existing knowledge. The recent methods of standard FSL can be divided into the following two categories: \textit{meta learning} and \textit{transfer learning}. Meta learning methods commonly train a meta-learner in an episodic way, where the base data is divided into multiples tasks each with a few samples and thereby class-agnostic knowledge is acquired from previous experiences \cite{FinnAL17, JamalQ19, LeeMRS19, SunLCS19, SnellSZ17, VinyalsBLKW16, XuXWT21, ZhangCLS20}.

% Then, this general knowledge can be utilized for the model to quickly be adapted to a few samples of novel classes (i.e., meta-test) by directly finding a meta-learner \cite{FinnAL17, JamalQ19, LeeMRS19, SunLCS19, BertinettoHTV19,  Nichol18,  RusuRSVPOH19} or performing metric-based training \cite{SnellSZ17, VinyalsBLKW16, XuXWT21, ZhangCLS20, iuCLL0LH20, Mangla0SKBK20}.
% 제거된 항목 meta: BertinettoHTV19,  Nichol18,  RusuRSVPOH19  & 아래 문항 transfer:  iuCLL0LH20, Mangla0SKBK20

% To this end, optimization-based approaches  perform gradient descent to directly find a meta-learner that can quickly be optimized to new classes by its good weight initialization. In metric-based meta learning \cite{ShiSSW20, SnellSZ17, VinyalsBLKW16, XuXWT21, YoonSM19, ZhangCLS20}, we train the meta-learner based on metric distances in embedding spaces while trying to adjust distances between feature vectors of query and support samples therein. 

Unlike meta learning, which requires a complex episodic training, transfer learning \cite{ChenLKWH19, DhillonCRS20, TianWKTI20, wang2019} simply re-uses the feature extractor learned from all collected base samples when fine-tuning only the classifier of the model being trained for FSL. These transfer learning approaches often show a better performance than meta learning particularly in deeper neural networks \cite{DhillonCRS20}, and therefore are getting more attention from the state-of-the-art methods. Although transfer learning seems to be quite effective in standard FSL, its performance is still unsatisfactory in generalized few-shot learning (GFSL) without leveraging additional techniques. However, this paper claims that our proposed normalization method can enable simple transfer learning to be highly effective in GFSL even without using any base samples.

% check 3. inform the original output of the pretrained model to be properly modified
\smalltitle{Generalized few-shot learning}
GFSL aims to handle both base and novel classes in a joint space, and hence is regarded to be more difficult than standard FSL. In GFSL, we not only need to learn as much new knowledge as possible but also preserve the pretrained knowledge over base classes. The major approach to this end in GFSL is not to entirely change a pretrained model as we do in transfer learning, but to train some extra architectures, which inform the original output of the pretrained model to be properly modified for a balanced inference between base and novel classes \cite{GidarisK18,RenLFZ19,YoonKSM20} occasionally by leveraging supplement information \cite{Li0LFLW20,ShiSSW20}.

% For example, \textit{DFSL} \cite{GidarisK18} trains a temporary module that directly generates the weights of the novel classifier, which is simply concatenated with the base classifier at inference time. An extra network is also used to regularize the novel classifier for avoiding its bias toward novel classes in \textit{AAN} \cite{RenLFZ19}. In addition, \textit{XtarNet} \cite{YoonKSM20} not only adds an architecture to aid the novel classifier, but also trains another auxiliary network that can learn feature representations of novel classes. 

The other approach without introducing any extra architecture is to fine-tune a pretrained model with a balanced dataset of base and novel classes, namely \textit{balanced fine-tuning}. In order to improve the performance, balanced fine-tuning is often performed together with additional techniques like \textit{weighting imprinting} of classifiers \cite{QiBL18} and introducing a three-step framework \cite{KuklevaKS21}. For a better balanced dataset, \textit{hallucination} approaches \cite{HariharanG17,WangGHH18} synthesize novel instances based on the base dataset.

\smalltitle{Incremental few-shot learning}
In continual learning, GFSL is being extended to \textit{incremental few-shot learning} (\textit{IFSL}) \cite{KuklevaKS21,MazumderSR21,TaoHCDWG20,ZhangSLZPX21}, in which the model has to go through a series of tasks of few-shot classes. Similar to the existing GFSL methods, most IFSL methods take the best use of the previously collected dataset to preserve the previous knowledge and learn the relationship among different classes.

All the existing works in GFSL and IFSL somehow need base data for either training extra architectures, balanced fine-tuning, or preserving the previous knowledge. To our best knowledge, this paper is the first study that resolves \textit{zero-base GFSL} via weight normalization without using base data. Although the existing GFSL methods \cite{FanMLS21,GidarisK18,QiBL18,WangH0DY20} also perform basic weight normalization using a cosine classifier, their normalization scheme alone fails to achieve a satisfactory performance without their additional training techniques relying on base data. In this paper, we investigate what occurs to weight distribution during GFSL, and thereby effectively overcome the limitation of the existing basic normalization method.

% related work
\section{Preliminary} \label{sec:prob}

% In this section, we start with the basic framework of generalized few-shot learning (GFSL) and then formally define our zero-base GFSL problem.

% Just mention about GFSL
% Notations
% Basic finetuning method (1. Pretraining, 2. finetuning)
% Limitation of this basic method
% Challenges...(we need data...)

\subsection{Basic Framework of GFSL}
In generalized few-shot classification, the training dataset is composed of a base dataset, denoted by ${\D}_{base} =  \{(\mathbf{x_i}, \mathbf{y_i})\}^{K_{base}}_{i=1}$, and a novel dataset, denoted by ${\D}_{novel} =  \{(\mathbf{x_j}, \mathbf{y_j})\}^{K_{novel}}_{j=1}$, where $(\mathbf{x}, \mathbf{y})$ is a training instance and $K_{base} \gg K_{novel}$ is usually assumed. We denote $C_{base}$ and $C_{novel}$ be two disjoint sets of base classes and novel classes, respectively, where $\abs{C_{base}} \gg \abs{ C_{novel}}$. 

Then, the goal of GFSL is to train a model $\Phi = \{\phi, \theta\}$ with $\D_{base} \cup \D_{novel}$ such that $\Phi$ can discriminate any classes in $C_{base} \cup C_{novel}$, which is different from standard FSL aiming to classify only $C_{novel}$. $\phi$ and $\theta$ are parameters of the feature extractor and the \textit{linear} classifier of the model, and we also denote $f_{\phi}(\mathbf{x})$ be the feature vector returned from $\phi$ given $\mathbf{x}$, where the dimensionality of each feature vector is denoted as $d$. We then make inference by taking the softmax of $\theta^\top f_\phi(\mathbf{x})$ for a given input sample $\mathbf{x}$.

The basic framework of GFSL consists of two stages, namely pre-training and fine-tuning. This two-stage transfer learning scheme, which fully utilizes the knowledge of base classes to learn few-shot novel classes, is regarded as one of the leading paradigms for GFSL as well as standard FSL due to its simplicity and effectiveness.

% Basic finetuning method (1. Pretraining, 2. finetuning)
\smalltitle{Pre-training} In the pre-training stage, we train a full model $\Phi = \{\phi, \theta_{base} \}$ only with $\D_{base}$ by:
\begin{equation}\label{eq:pretraining}
     \Phi = \argmin_{\{\phi,\theta_{base}\}} \sum_{(\mathbf{x},\mathbf{y}) \in {\D}_{base}} -\mathbf{y} \log p(\mathbf{x}) + \R (\phi, \theta_{base}),
\end{equation}
where $\theta_{base} \in \mathbb{R}^{d \times \abs{C_{base}}}$ is the base classifier, $p(\mathbf{x})$ is the softmax output probability, i.e., $p_i(\mathbf{x}) = \frac{exp(\mathbf{z}_i)}{\sum_{j=1}^{|C_{base}|}exp(\mathbf{z}_j)}$ such that $\mathbf{z} = \theta^\top f_\phi(\mathbf{x})$, and $\R$ is the regularization term for the full model. This stage is equivalent to typical supervised learning as $\D_{base}$ has a sufficient number of samples.

\smalltitle{Fine-tuning}
Once a model is well-trained over base classes, we then fine-tune the model with respect to $C_{base} \cup C_{novel}$. In order to incorporate novel classes, the model should be extended to $\Phi = \{\phi, \theta \}$ in this fine-tuning stage, where $\theta=\{\theta_{base},\theta_{novel}\}$, $\theta_{base}$ is the trained base classifier and $\theta_{novel}$ is the novel classifier randomly initialized. Also, in order to make all the classes well balanced in the fine-tuned model, existing GFSL methods \cite{KuklevaKS21,QiBL18} train a balanced dataset, which is either a subset or superset of $\D_{base} \cup \D_{novel}$, constructed by undersampling $\D_{base}$ or oversampling $\D_{novel}$. Given such a balanced dataset, most state-of-the-art methods \cite{RenLFZ19,YoonKSM20} fine-tune only the classifier while freezing the feature extractor. This is due to the fact that the feature extractor is already well-trained on a sufficient number of samples and hence better to be frozen not to hurt generalization of the model.

% Limitations of basic framework
As long as a balanced dataset is ideally constructed, the model can well be trained for both base and novel classes by this basic framework as fine-tuning can properly rearrange all decision boundaries in the joint feature space. Unfortunately, vanilla sampling strategies \cite{HuangLLT16, WangRH17} would not work well in GFSL as undersampling inevitably causes information loss and oversampling inherently suffers from the lack of diversity due to many redundant samples. This is why many existing works propose extra techniques to achieve a better performance, which is not quite satisfactory according to our experiments. More importantly, a balanced dataset can be collected and fine-tuned only if the base dataset is available, which may not always be the case.

\subsection{Problem Statement of Zero-Base GFSL} % zero-base GFSL

In the proposed zero-base GFSL problem, our objective is to train a joint linear classifier $\theta$, which works well for both $C_{base}$ and $C_{novel}$, with only a few samples of novel classes (i.e., $\D_{novel}$) without using any samples of $\D_{base}$. Following the basic GFSL framework, we only fine-tune the classifier while freezing the feature extractor as:
\begin{equation}\label{eq:finetuning}
   \theta = \argmin_\theta \sum_{(\mathbf{x},\mathbf{y}) \in {\D}_{novel}} -\mathbf{y} \log p(\mathbf{x})+ \mathcal{R} (\theta).
   %\theta = \argmin_\theta \sum_{(x,y) \in {\D}_{novel}} -y \log \frac{exp(\theta^\top f_\phi(x))}{\sum_{c \in C} exp(\theta_c^\top f_\phi(x))} + \mathcal{R} (\theta),
\end{equation}
Note that it is even more desired to freeze the feature extractor in zero-base GFSL than in GFSL because the feature extractor can severely be biased toward novel classes after fine-tuning with only novel samples.

% The model is expected to classify in the test set with class in $C = C_{base} \cup C_{novel}$. Base classes and Novel classes are disjoint, i.e. $C_{base} \cap C_{novel} = \emptyset$ and we assume that $K_{base} \gg K_{novel}$ and $\abs{C_{base}} \gg \abs{ C_{novel}}$ in the GFSL classification task. In contrast, standard few-shot classification task only expect to classify objects in the dataset with typical known $N$-way $K$-shot classification, which is $N = \abs{C_{novel}}$ and $K = K_{novel}$. Basic formulation of GFSL is two-stage framework with simple fine-tuning.

% losing the learned capability from the zero-base class samples $\D_{base}$.

% the few novel class samples $\D_{novel}$ without forgetting the learned capability from the zero-base class samples $\D_{base}$. We only fine-tune the joint linear classifier $\theta$, while freezing the feature extractor $f(\cdot)$:

% note that we do not consider fine-tuning the entire pre-trained model because freezing feature extractors in the fine-tuning stage for zero-base GFSL, the performance for base classes is drastically degraded.
%as observed in \cite{KuklevaKS21}.
 % prob. def

\begin{figure}[t]
	\centering
	\includegraphics[height=1.8mm]{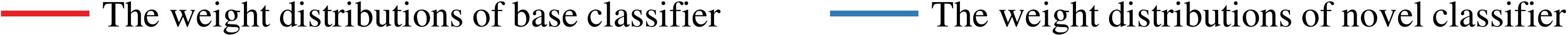}
    \subfigure[\label{fig:dist:a}1-shot]{\hspace{1mm}\includegraphics[width=0.4
    \columnwidth]{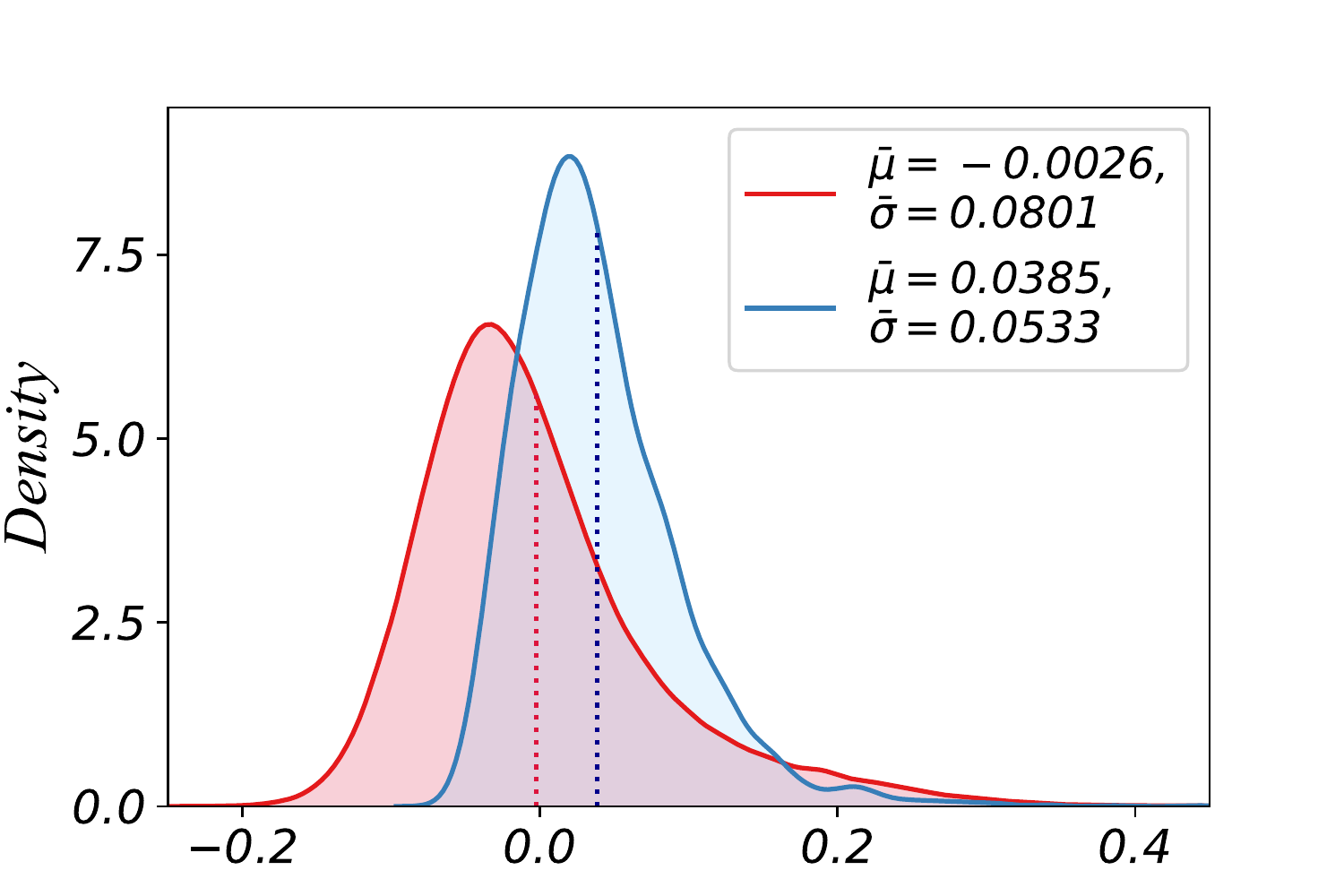}\hspace{1mm}} 
    \subfigure[\label{fig:dist:b}5-shot]{\hspace{1mm}\includegraphics[width=0.4\columnwidth]{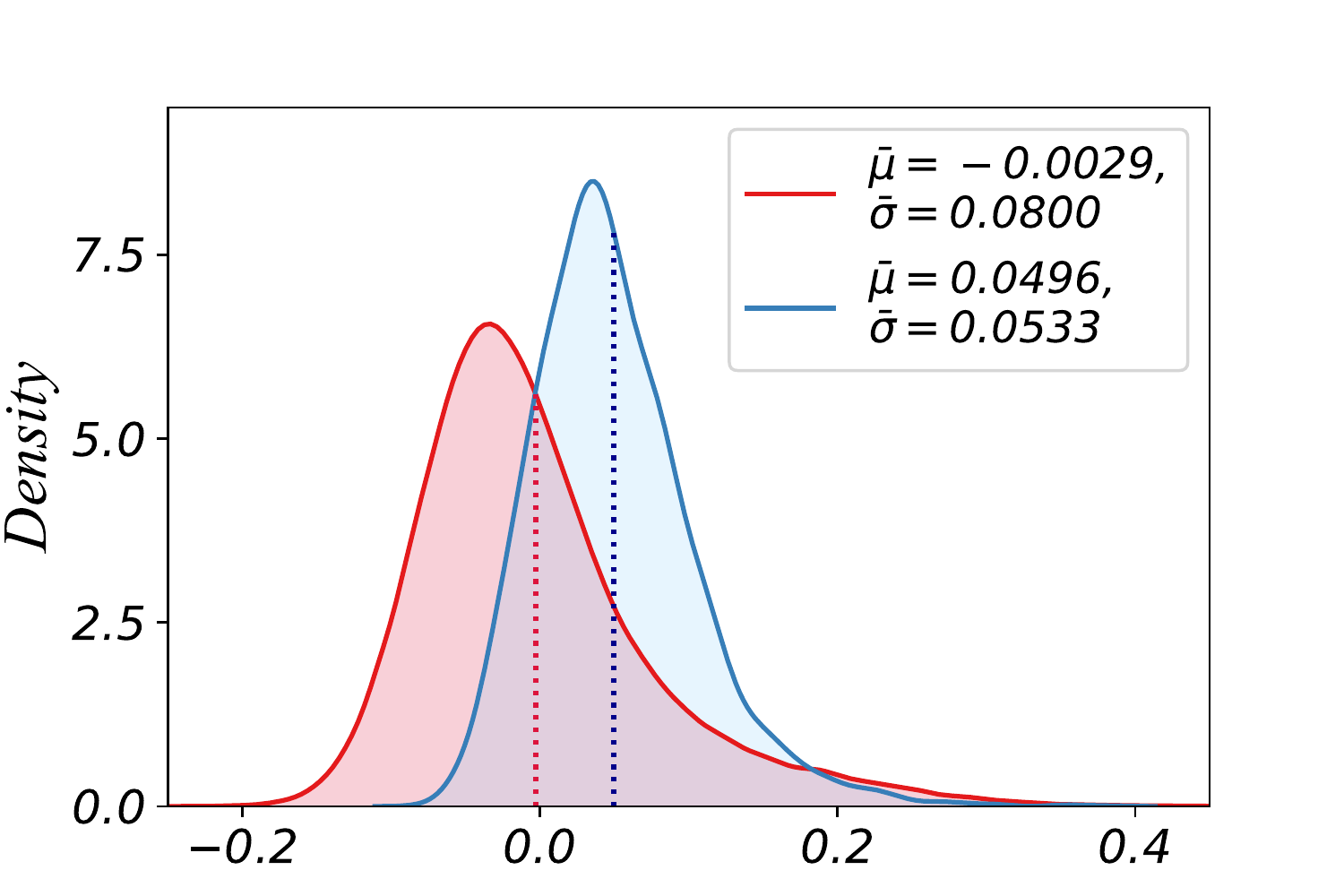}\hspace{1mm}}
    % \subfigure[\label{fig:dist:c}10-shot]{\hspace{1mm}\includegraphics[width=0.3\columnwidth]{./figure/figure_dist_10shot.pdf}\hspace{1mm}} 
    \caption{Weight distributions of base and novel classifiers where dotted lines represent means, using \textit{tiered-ImageNet} with ResNet-18.}
	\label{fig:shifting}
	\vspace{-2mm}
\end{figure}

\section{Analysis and Methodology on Zero-Base GFSL} \label{sec:anaylsis}

This section first investigates why the basic framework of GFSL cannot perform well in zero-base GFSL, and then presents an effective weight normalization method that enables a simple transfer learning scheme to achieve even a better performance than existing GFSL methods assuming the existence of base data.

\subsection{Hardness of Zero-Base GFSL}
As presented in Figure \ref{fig:imagenet}, a model trained by the basic framework of GFSL, that is, fine-tuning only the joint classifier $\theta$ with $\D_{novel}$, turns out to be extremely inaccurate for base classes yet pretty accurate for novel classes despite freezing the feature extractor. Thus, the fine-tuned classifier gets highly biased toward novel classes due to the absence of base samples. 

More specifically, when freezing the feature extractor $\phi$, the feature space itself is fixed and so are the feature vectors $f_\phi(\mathbf{x})$ of all the samples $\mathbf{x}$. Consequently, it is 
decision boundaries in the space that are changed by fine-tuning the classifier and therefore matter to a biased prediction. Without a balanced dataset, it is challenging to properly learn hidden relationships between base and novel classes, leading to \textit{enlarged} (i.e., biased) decision boundaries of novel classes as depicted in Figure \ref{fig:decision}. In zero-base GFSL, it is our mission to find the best balance between decision boundaries of base and novel classes without retraining any base samples so that we can ideally make the best joint prediction. As decision boundaries between base and novel classes are formed where $\theta_{base}^\top f_\phi(\mathbf{x}) = \theta_{novel}^\top f_\phi(\mathbf{x})$, they are mainly affected by how the weight values of $\theta_{base}$ and $\theta_{novel}$ are distributed. In the following subsection, we therefore conduct a systematic analysis on the weight distributions of $\theta_{base}$ and $\theta_{novel}$.

\subsection{Analysis on Weight Distributions of Classifiers} 

% We investigate the distribution of weights - mean & variance (definition)
We investigate the underlying distributions of weights of base and novel classifiers particularly in terms of their mean $\mu$ and variance $\sigma^2$. More specifically, given a classifier $\theta = [\theta_1, \theta_2, \ldots, \theta_{\abs{C}}] \in \mathbb{R}^{d \times \abs{C}}$, where $\theta_i \in \mathbb{R}^{d} $ represents the weight vector corresponding to class $i$, we first define the vectors of class-wise means and standard deviations of weight values as follows:
\begin{equation} \label{eq:map}
  \mu = 
  [\mu_1,\mu_2,\ldots,\mu_{\abs{C}}] \text{ and } \sigma = [\sigma_1,\sigma_2,\ldots,\sigma_{\abs{C}}], \nonumber \\
\end{equation}
where $\mu_i = \frac{1}{d} \sum_{j=1}^{d}{\theta_{i, j}}$ and $\sigma_i^2 = \frac{1}{d} \sum_{j=1}^{d}{(\theta_{i, j}-\mu_i)^2}$. Also, we use $\mu_{base}$, $\sigma_{base}$, $\mu_{novel}$, and $\sigma_{novel}$ to denote the vectors corresponding to the weights of base and novel classifiers (i.e., $\theta_{base}$ and $\theta_{novel}$). 

In order to examine how the weight values of the novel classifier are different from those of the base classifier, we compute the average of class-wise means and standard deviations as $\bar \mu = \frac{1}{\abs{C}} \sum_{i=1}^{\abs{C}}{\mu_i} \text{ and } \bar \sigma = \frac{1}{\abs{C}} \sum_{i=1}^{\abs{C}}{\sigma_i}$ and compare $\bar \mu_{novel}$ and $\bar \sigma_{novel}$ with $\bar \mu_{base}$ and $\bar \sigma_{base}$. As observed in Figure \ref{fig:shifting}, we discover the following two undesirable facts, both of which should be tackled to obtain a balanced joint classifier.

% and we compute the class-wise average and standard deviations $\bar \mu = \frac{1}{\abs{C}} \sum_{i=1}^{\abs{C}}{\mu_i} $ and $\bar \sigma = \frac{1}{\abs{C}} \sum_{i=1}^{\abs{C}}{\sigma_i} $.

\smalltitle{L2-norms proportional to standard deviations}
We first discover that the variance of $\theta_{base}$ is greater than that of $\theta_{novel}$, i.e., $\bar \sigma_{base} > \bar \sigma_{novel}$, as shown in Figure \ref{fig:shifting}. This is somewhat intuitive in that the base classifier is trained on a large number of samples in $\D_{base}$ that can increase the diversity, whereas the novel classifier cannot have such a large diversity due to the lack of training samples in $\D_{novel}$. What is \eat{maybe change}not quite obvious is that these standard deviations are indeed proportional to the L2-norms of weights in neural networks as the following proposition.
\begin{proposition} \label{thm:std}
Consider a parameter $\theta$ of $N$ weights in a neural network with the assumption that $\theta$ is randomly initialized by a Gaussian distribution with zero mean and trained by a uniformly distributed dataset. Then, it holds that $\sigma  = {\rVert \theta \rVert_2 \cdot \frac{1}{\sqrt{N}}}$.
\end{proposition}
\begin{proof}
Given $E[\theta] = 0$, we have: $\sigma = \sqrt{E[\theta-E(\theta)]^2} = \sqrt{E[\theta^2]} \approx \sqrt{\frac{1}{N}\sum_{i=1}^{N}\theta_i^2} = {\frac{1}{\sqrt{N}} \rVert \theta \rVert _2}.$
\end{proof}
By Proposition \ref{thm:std}, $\rVert\theta_{base}\rVert_2 > \rVert\theta_{novel}\rVert_2$ is implied from the observation of $\bar \sigma_{base} > \bar \sigma_{novel}$. Somewhat surprisingly, this contradicts to the highly biased results toward novel classes in Figure \ref{fig:imagenet} in the sense that many existing works \cite{HouPLWL19,KangXRYGFK20,ZhaoXGZX20} report that larger weight norms of a class tend to produce larger logits leading to more biased predictions toward the corresponding class. Thus, if we follow the existing strategy of equalizing the weight norms \cite{HouPLWL19,KangXRYGFK20,ZhaoXGZX20} by increasing the smaller ones yet decreasing the larger ones, the novel-biased classifier in zero-base GFSL would get even more biased toward novel classes, to be confirmed in our experiments.

Then, how could the classifier be highly biased toward novel classes even though the weight norms of novel classes are less than those of base classes? The true answer to this question lies in the average weight of each classifier rather than its weight norm.

\smalltitle{Mean shifting phenomenon}
In terms of means of weight distributions, we observe the fact that the average weight of the novel classifier is positively shifted (i.e., $\bar \mu_{novel} > 0$) while that of the base classifier keeps almost zero (i.e., $\bar \mu_{base} \approx 0$). We call this situation \textit{mean shifting phenomenon} commonly observed in all the graphs of Figure \ref{fig:shifting}. 

Positively shifted $\bar \mu_{novel}$ indicates that each weight of $\theta_{novel}$ on the average has a more positive value. Recalling that decision boundaries are constructed at $\theta_{base}^\top f_\phi(\mathbf{x}) = \theta_{novel}^\top f_\phi(\mathbf{x})$, more positive weights of $\theta_{novel}$ would increase $\theta_{novel}^\top f_\phi(\mathbf{x})$, and therefore enlarge the decision boundaries of novel classes. This implies that mean shifting phenomenon is a more critical reason behind the novel-biased model in zero-base GFSL. As mentioned above, however, the existing works \cite{HouPLWL19,KangXRYGFK20,ZhaoXGZX20} only focus on equalizing the weight norms, which turn out to be proportional to the standard deviations of weights, between base and novel classes, and hence their normalization method alone fails to achieve a satisfactory performance.

% new figure in aaai
% figure 4. Probability 
\begin{figure}[t!]
	\centering
    \includegraphics[width=0.98\columnwidth]{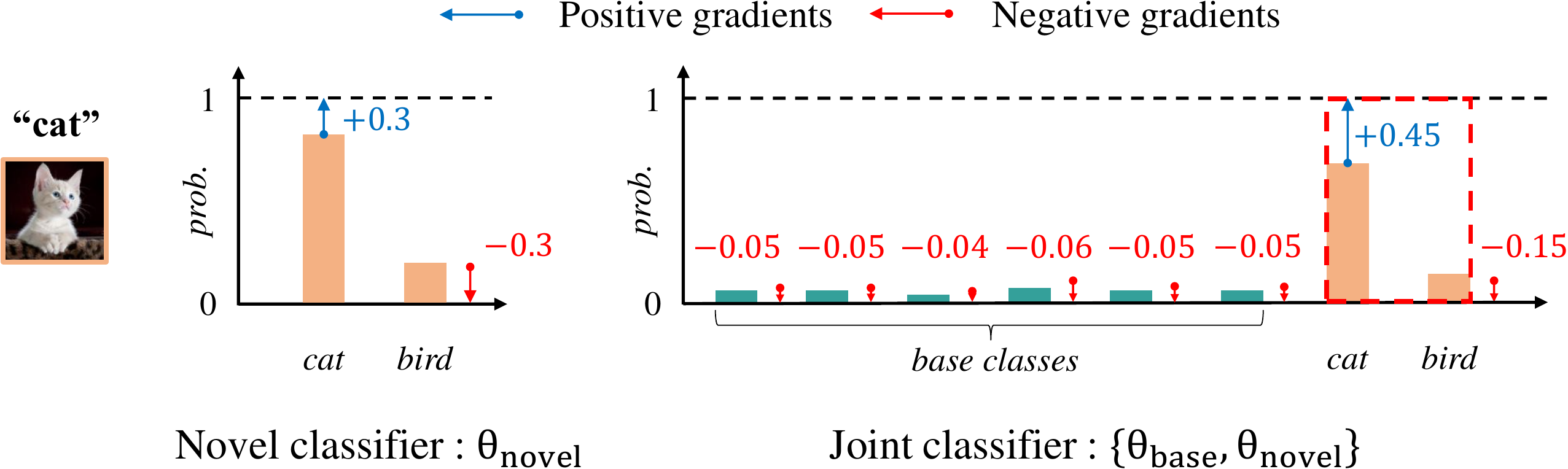}
    \caption{An illustrative example for the positively shifted mean of weights for novel classes, where a cat image gets probabilities $[0.7~0.3]$ from $\theta_{novel}$ but its probabilities becomes $[0.05~0.05~0.04~0.06~0.05~0.05~|~0.55~0.15]$ in $\theta = \{\theta_{base}, \theta_{novel}\}$.}
    \label{fig:prob}%Decision boundaries in the feature space
\end{figure}

% Why mean shifting occurs? 
\smalltitle{What makes mean shifted}
Let us now analyze why mean shifting phenomenon occurs to the novel classifier in zero-base GFSL. To this end, we first consider how we update the weights of each classifier by gradient descent as $\theta = \theta - \eta {\partial \L \over \partial \theta }(f_\phi(\mathbf{x}),\theta)$, where the gradient is:
\begin{equation} \label{eq:gradient}
    \begin{aligned}
    %{\partial  \over \partial \theta_i } \L (f_\phi(x),\theta_i) &= {\partial \L \over \partial z} {\partial z \over \partial \theta_i }  \\
    {\partial  \over \partial \theta } \L (f_\phi(\mathbf{x}),\theta) = {\partial \L \over \partial z}\cdot{{\partial z \over \partial \theta}} = f_\phi(\mathbf{x}) (\mathbf{y} - p(\mathbf{x})).
    \end{aligned}
\end{equation}
Since $f_\phi(\mathbf{x})$ is always non-negative due to the ReLU function and $p(\mathbf{x})$ is also a non-negative probability value, the gradient is positive only if $\mathbf{y}_i = 1$ (i.e., for the probability of the class label of $\mathbf{x}$) and negative for all the other class probabilities. In training a balanced dataset of all the classes uniformly distributed, the total positive gradient for each class eventually gets similar to the absolute value of its total negative gradient, which is why $\bar \mu \approx 0$. Thus, if we train $\D_{novel}$ only to  $\theta_{novel}$, $\mu_{novel}$ would also be close to zero. 

However, when fine-tuning the joint classifier $\theta = \{\theta_{base}, \theta_{novel}\}$ with only $\D_{novel}$, the total positive gradient of $\theta_{novel}$ increases whereas the absolute value of its total negative gradient decreases because all the output probabilities of $\theta_{novel}$ together become smaller in training $\theta$ than they used to be in training only $\theta_{novel}$. To illustrate, consider an example of Figure \ref{fig:prob}, where output probabilities of $\theta_{novel}$ are larger than those of $\theta = \{\theta_{base}, \theta_{novel}\}$ for a given training image of \textit{`cat'}. Thus, even if $\theta_{novel}$ outputs $[0.7~~0.3]$ for the cat image, the joint classifier $\theta$ would output something like $[0.05~~0.05~~0.04~~0.06~~0.05~~0.05~|~0.55~~0.15]$ for 6 base classes followed by 2 novel classes. Consequently, the positive gradient caused by the novel class \textit{`cat'} is increased from $0.3$ to $0.45$, but the absolute value of the negative gradient for the second novel class \textit{`bird'} is decreased from $|-0.3|$ to $|-0.15|$. Furthermore, $\theta_{novel}$ can receive negative gradients only from the instances of the other novel classes due to the absence of base instances. This makes $\mu_{novel}$ to be more positively shifted. 

As for $\theta_{base}$, only negative gradients will arrive during the entire fine-tuning process, but its total amount would not be that much because the well-trained $\theta_{base}$ is not likely to be overconfident to irrelevant novel classes particularly when $\theta_{novel}$ gets used to identifying their corresponding novel classes. This is why $\mu_{base}$ still keeps to be zero after fine-tuning.

\subsection{Solution: Mean-Variance Classifier Normalization} \label{sec:solution}

\begin{figure*}[t!]
   \centering 
    \includegraphics[width=1.7\columnwidth]{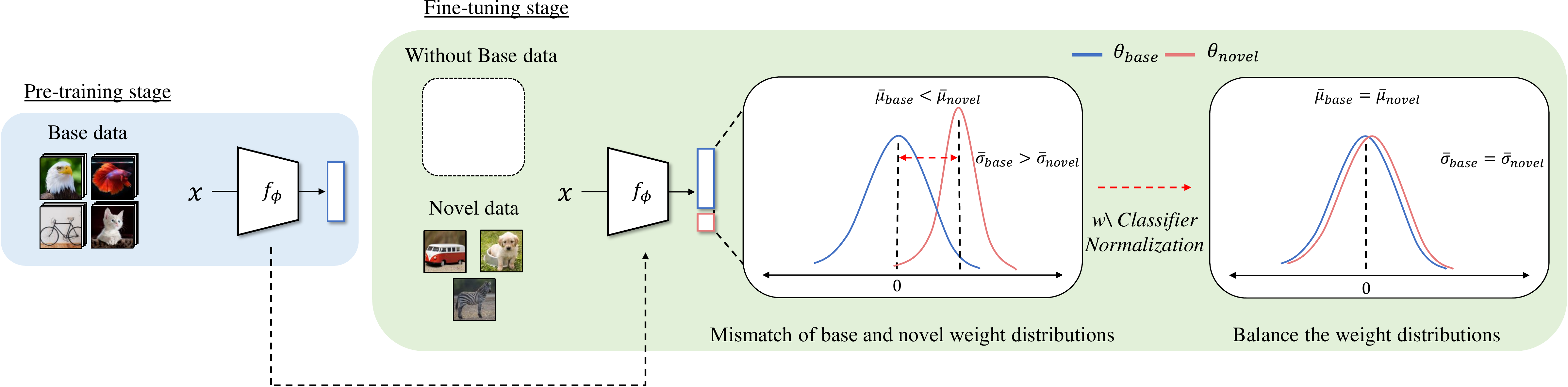}
   \caption{Overview of Mean-Variance Classifier Normalization. In the pre-training stage, we train the entire model on all base classes. In the fine-tuning stage, we normalize the novel classifier by online mean centering in the process of training and adjust the trained weights of the base and novel classifiers by re-scaling them using the standard deviation ratio.}
   \label{fig:framework}
\end{figure*}

To fulfill both the zero-mean and a balanced variance of base and novel classifiers, we propose a simple yet effective normalization method, called \textit{Mean-Variance Classifier Normalization} (\textit{MVCN}). MVCN takes two essential steps, namely \textit{online mean centering} and \textit{offline variance balancing}. We perform online mean centering during fine-tuning the joint classifier, and then proceed to balance variances of weights as a post processing step once the classifier is well trained. The entire process is outlined in Figure \ref{fig:framework}.

\smalltitle{Online mean centering} First, we keep normalizing the weights of the novel classifier to have zero-mean in the process of training by:
\begin{equation}
\begin{aligned}
  \hat \theta_{novel} &= \theta_{novel} - \mu_{novel}. \\
  %\textit{where } \bar \theta_i &= \sum_{j\in d} \theta_{i,j}
\end{aligned}
\end{equation}
When fine-tuning the joint classifier, online mean centering is performed only on the novel classifier. This directly reduces the positive weights of the novel classifier and keeps the zero-mean while learning the features of novel classes. Note that we do not give any constraints to the variance of weights during fine-tuning, but rather allow each classifier to learn as many required features as possible.

%  But, we also give free to variance of joint classifier in training because we do not limit learning the features of novel classes. We does not consider the mean-centering of the base classifier, which is to avoid ruining the base classifier that has already been learned with abundant, and there is no problem with the joint prediction by receiving negative gradients from the novel samples.

% As shown in theorem \ref{thm:std}, we do not directly normalize the weights vector by std, because 
\smalltitle{Offline variance balancing} 
Once the fine-tuning stage is done, we adjust the weights of base classifier according to the ratio of standard deviations as follows:
\begin{equation}
\begin{aligned}
  \hat \theta_{base} &= {\bar \sigma_{novel} \over \sigma_{base}} \cdot \theta_{base}, \\
\end{aligned}
\end{equation}
where $\bar \sigma_{novel}$ is the average of class-wise standard deviations for all novel classes. We re-scale each weight of the base classifier by multiplying the ratio of the standard deviation of the novel classifier to that of the base classifier (i.e., $\bar \sigma_{novel} \over \sigma_{base}$), and thereby the weight variance of the base classifier gets similar to that of the novel classifier. Note that we do not directly normalize $\theta$ by their $\sigma$, which makes the zero-mean and unit-variance ${\theta - \mu \over \sigma}$, because the standard deviation of weights is much smaller than 1.

% Instead of adjusting the scale of the novel classifier learned with very few data, we can adjust the generalized base classifier learned with a lot of data, so that we can achieve the best joint prediction performance. Note that, we do not directly normalize $\theta$ by their $\sigma$, which makes the zero-mean and unit-variance ${\theta_i - \mu_i \over \sigma_i}$, because the std of weights is smaller than 1 $\sigma \ll 1$.

% Figure \ref{fig:framework} summarizes the entire process of our zero-base GFSL method, where our proposed Classifier Normalization enables the model being trained to achieve a balanced weight variance as well as the zero-mean of the classifier weights, which results in a well-balanced joint prediction.

\smalltitle{Post linear optimization without base data}
Finally, we introduce class-wise learnable parameters $\gamma_i$ and $\beta_i$ for $i \in C_{base} \cup C_{novel}$, which are intended for further optimizing decision boundaries of novel classes. With freezing $\theta = \{\theta_{base}, \theta_{novel}\}$, $\gamma_i$ and $\beta_i$ are similarly trained by Eq. (\ref{eq:finetuning}) and used at inference time as $\gamma \cdot \theta^\top f_\phi(\mathbf{x}) + \beta$. Through these additional parameters, we can further improve the performance of novel classes particularly in an extreme case of 1-shot learning. Note that this optimization process is performed still without using any base samples.

% main
\section{Experiments} \label{sec:ex}

\subsection{Experimental Settings}

\smalltitle{Datasets}
We compare our method with the state-of-the-art (SOTA) GFSL methods using two datasets, \textit{mini-ImageNet} \cite{VinyalsBLKW16} and \textit{tiered-ImageNet} \cite{RenTRSSTLZ18}, which are most widely used in the literature of GFSL. The \textit{mini-ImageNet} contains 100 classes and 60,000 sample images from \textit{ImageNet} \cite{RussakovskyDSKS15}, which are then randomly split into 64 training classes, 16 validation classes, and 20 testing  classes, proposed by \cite{RaviL17}. The \textit{tiered-ImageNet} is another subset of \textit{ImageNet}, containing 608 classes, which are then split into 351 training, 97 validation, 160  testing classes. This setting is more challenging since base classes and novel classes come from different super classes. The size of all images in both datasets is $84 \times 84$. In addition, we test whether our method works better than the SOTA \textit{GZSL} (\textit{Generalized Zero-Shot Learning}) methods using the three most common datasets, \textit{CUB} \cite{welinder2010caltech}, \textit{AWA1} \cite{LampertNH09}, and \textit{AWA2} \cite{XianLSA19}.

\smalltitle{Implementation details}
We implement all the methods in PyTorch, and train each model on a machine with NVIDIA A100. We use ResNet-12 \cite{HeZRS16} for \textit{mini-ImageNet} and ResNet-18 for \textit{tiered-ImageNet}, according to \cite{RenLFZ19, YoonKSM20}.  For \textit{CUB}, \textit{AWA1} and \textit{AWA2}, we commonly use ResNet-101. Full details of our settings are covered in Appendix.

% .

\eat{During the test time, we evaluate our method by taking an average over 600 randomly sampled tasks. For all the experiments, we consider 5-way classification with 1, 5, and 10 novel train samples (i.e. shots). The 10-shot task has never been tested in the literature of GFSL, and hence we report all the results for 10-shot tasks obtained on our own by executing the implementations of the compared methods. For XtarNet \cite{YoonKSM20}, we reproduce all the results by the authors' implementation, and the rest of the results are referred from \cite{KuklevaKS21}. Full details are covered in our supplementary material.}
% We present the overall accuracy `ALL', which is computed between base and novel prediction in the joint space and we report performance of `Novel' and `Base' in the joint space of base and novel classes. 

\subsection{Experimental Results}

\begin{table*}[]
\renewcommand{\arraystretch}{1.4}
\centering
\resizebox{0.87\textwidth}{!}{%
\begin{tabular}{l|ccc|ccc|ccc}
\toprule 
\multirow{2}{*}[-0.5em]{\textbf{Methods/ Shots}}                                           & \multicolumn{3}{c|}{\textbf{1-shot}}                                & \multicolumn{3}{c|}{\textbf{5-shot}}                               & \multicolumn{3}{c}{\textbf{10-shot}}                               \\ \cmidrule(r){2-10}
                                                                                   & \textbf{Novel} & \multicolumn{1}{c||}{\textbf{Base}} & \textbf{All}  & \textbf{Novel} & \multicolumn{1}{c||}{\textbf{Base}} & \textbf{All} & \textbf{Novel} & \multicolumn{1}{c||}{\textbf{Base}} & \textbf{All} \\ \cmidrule(r){1-10}
GcGPN \cite{ShiSSW20}                                                                             & 39.86          & \multicolumn{1}{c||}{54.65}         & 47.25         & 56.32          & \multicolumn{1}{c||}{59.30}         & 57.81        & -              & \multicolumn{1}{c||}{-}             & -            \\
IW    \cite{QiBL18}                                                                             & 41.32          & \multicolumn{1}{c||}{58.04}         & 49.68         & 59.27          & \multicolumn{1}{c||}{58.68}         & 58.98        & 45.85         & \multicolumn{1}{c||}{72.53}        & 59.19 $\pm$ 0.19 \\
DFSL  \cite{GidarisK18}                                                                             & 31.25          & \multicolumn{1}{c||}{17.72}         & 39.49         & 46.96          & \multicolumn{1}{c||}{58.92}         & 52.94        & 66.04         & \multicolumn{1}{c||}{69.87}        & 67.95 $\pm$ 0.17        \\
AAN  \cite{RenLFZ19}                                                                              & 45.61          & \multicolumn{1}{c||}{63.92}         & 54.76         & 60.82          & \multicolumn{1}{c||}{64.14}         & 62.48        & 66.33         & \multicolumn{1}{c||}{62.49}        & 64.41 $\pm$ 0.16 \\
LCwoF      \cite{KuklevaKS21}                                                                        & 53.78          & \multicolumn{1}{c||}{62.89}         & 57.84         & 68.58          & \multicolumn{1}{c||}{64.53}         & 66.55        & 76.71 $\pm$ 0.58    & \multicolumn{1}{c||}{62.86 $\pm$ 0.07}   & 69.78 $\pm$ 0.28  \\
XtarNet   \cite{YoonKSM20}                                                                         & 47.04 $\pm$ 0.29  & \multicolumn{1}{c||}{64.17 $\pm$ 0.22} & 55.61 $\pm$ 0.17 & 62.46 $\pm$ 0.25   & \multicolumn{1}{c||}{\textbf{70.57} $\pm$ \textbf{0.20}}    & 66.52 $\pm$ 0.16 & 70.46 $\pm$ 0.23   & \multicolumn{1}{c||}{\textbf{70.88} $\pm$ \textbf{0.21}}  & 70.67 $\pm$ 0.15 \\ \cmidrule(r){1-10}
\begin{tabular}[c]{@{}c@{}} MVCN (ours)\end{tabular} & \textbf{51.72} $\pm$ \textbf{0.64}    & \multicolumn{1}{c||}{\textbf{65.81} $\pm$ \textbf{0.08}}   & \textbf{58.77} $\pm$ \textbf{0.35}   & \textbf{75.20} $\pm$ \textbf{0.46}     & \multicolumn{1}{c||}{67.62 $\pm$ 0.09}   & \textbf{71.22} $\pm$ \textbf{0.26}  & \textbf{78.7} $\pm$ \textbf{0.51}     & \multicolumn{1}{c||}{68.06 $\pm$ 0.09}   & \textbf{73.38} $\pm$ \textbf{0.25}  \\ 
\bottomrule
\end{tabular}
}
% \vspace{-1.5mm}
% % \caption{\textbf{Comparison on \textit{mini-ImageNet}.} Average few-shot classification accuracy (\%) with 95\% confidence intervals.} \label{tab:mini}
% \vspace{1.5mm}
% \end{table*}

% \begin{table*}[]
\renewcommand{\arraystretch}{1.4}
\centering
\resizebox{0.87\textwidth}{!}{%
\begin{tabular}{l|ccc|ccc|ccc}
\toprule 
\multirow{2}{*}[-0.5em]{\textbf{Methods/ Shots}}                                           & \multicolumn{3}{c|}{\textbf{1-shot}}                                & \multicolumn{3}{c|}{\textbf{5-shot}}                               & \multicolumn{3}{c}{\textbf{10-shot}}                               \\ \cmidrule(r){2-10}
                                  & \textbf{Novel} & \multicolumn{1}{c||}{\textbf{Base}} & \textbf{All} & \textbf{Novel} & \multicolumn{1}{c||}{\textbf{Base}} & \textbf{All} & \textbf{Novel} & \multicolumn{1}{c||}{\textbf{Base}} & \textbf{All} \\ \cmidrule(r){1-10}
IW    \cite{QiBL18}                                 & 44.95          & \multicolumn{1}{c||}{62.53}         & 53.74           & 71.85          & \multicolumn{1}{c||}{56.11}         & 63.98           & 74.01         & \multicolumn{1}{c||}{61.67}        & 65.50 $\pm$ 0.16    \\
DFSL  \cite{GidarisK18}                              & 47.32          & \multicolumn{1}{c||}{36.10}         & 41.71           & 67.94          & \multicolumn{1}{c||}{39.08}         & 53.51           & 73.97          & \multicolumn{1}{c||}{56.94}         & 65.46 $\pm$ 0.17    \\
AAN  \cite{RenLFZ19}                               & 54.39          & \multicolumn{1}{c||}{55.85}         & 55.12           & 57.76          & \multicolumn{1}{c||}{64.13}         & 60.95          & 70.88         & \multicolumn{1}{c||}{57.49}        & 64.18 $\pm$ 0.17    \\
LCwoF      \cite{KuklevaKS21}                             & 57.13          & \multicolumn{1}{c||}{60.39}         & 58.76           & 69.05          & \multicolumn{1}{c||}{63.44}         & 66.25           & 79.20 $\pm$ 0.72    & \multicolumn{1}{c||}{61.76 $\pm$ 0.10}   & 70.57 $\pm$ 0.41     \\
XtarNet    \cite{YoonKSM20}                       & 58.90 $\pm$ 0.29   & \multicolumn{1}{c||}{\textbf{64.02} $\pm$ \textbf{0.22}}  & 61.46 $\pm$ 0.18     & 74.49 $\pm$ 0.24   & \multicolumn{1}{c||}{63.13 $\pm$ 0.21}  & 68.81 $\pm$ 0.16    & 78.36 $\pm$ 0.23    & \multicolumn{1}{c||}{63.08 $\pm$ 0.22}   & 70.72 $\pm$ 0.16     \\ \cmidrule(r){1-10}
MVCN (ours)                      & \textbf{62.11} $\pm$ \textbf{0.70}     & \multicolumn{1}{c||}{61.23 $\pm$ 0.22}   & \textbf{61.67} $\pm$ \textbf{0.31}           & \textbf{79.59} $\pm$ \textbf{0.55}    & \multicolumn{1}{c||}{\textbf{66.3} $\pm$ \textbf{0.20}}     & \textbf{72.34} $\pm$ \textbf{0.28}           & \textbf{83.06} $\pm$ \textbf{0.56}    & \multicolumn{1}{c||}{\textbf{67.46} $\pm$ \textbf{0.16}}   & \textbf{75.26} $\pm$ \textbf{0.27}     \\ 
\bottomrule
\end{tabular}%
}
\vspace{-1.5mm}
\caption{Performance (\%) comparison with GFSL methods on \textit{mini-ImageNet} (top) and \textit{tiered-ImageNet} (bottom).} \label{tab:tiered}
\vspace{1.5mm}
% \end{table*}

%  \begin{table*}[]
\renewcommand{\arraystretch}{1.4}
\centering
\resizebox{0.82\textwidth}{!}{%
\begin{tabular}{l|cccc|cccc|cccc}
\toprule 
\textbf{Datasets}                                           & \multicolumn{4}{c|}{\textbf{CUB}}                                & \multicolumn{4}{c|}{\textbf{AWA1}}                               & \multicolumn{4}{c}{\textbf{AWA2}}                               \\ \cmidrule(r){1-13}
                                \textbf{Methods/ Shots}  & \textbf{1-shot} & {\textbf{2-shot}} & \textbf{5-shot} & \textbf{10-shot} & {\textbf{1-shot}} & \textbf{2-shot} & \textbf{5-shot} & {\textbf{10-shot}} & \textbf{1-shot} & \textbf{2-shot} & \textbf{5-shot} & \textbf{10-shot}\\ \cmidrule(r){1-13}
ReViSE    \cite{TsaiHS17}                                 & 36.3          & 
41.1      & 44.6           & 50.9          & 56.1         & 60.3           & 64.1         & 67.8        & -      & -     & -      & -       \\
CA-VAE  \cite{SchonfeldESDA19}                              & 50.6             & 
54.4         & 59.6           & 62.2          & 64.0             & 71.3           & 76.6          & 79.0          & 41.8           & 52.7          & 66.5   & 76.7      \\
DA-VAE  \cite{SchonfeldESDA19}                               & 49.2           & 54.6        & 58.8          & 60.8            & 68.0           & 73.0          & 75.6           & 76.8          & 68.6         & 77.1            & 81.8             & 81.3          \\
CADA-VAE      \cite{SchonfeldESDA19}                             & 55.2           & 59.2          & 63.0            & 64.9           & 69.6          & 73.7           & 78.1   & 80.2    & 73.6         &  78.9        & 81.9       & 85.0           \\
DRAGON    \cite{SamuelAC21}                       & 55.3       & 59.2   & 63.5      & \textbf{67.8}     & 67.1   & 69.1     & 76.7     & 81.9     & -      & -     & -      & -           \\ \cmidrule(r){1-13}
MVCN (ours)                      & \textbf{57.3}     & \textbf{61.6}    & \textbf{65.4}           & \textbf{67.8}    & \textbf{69.9}     & \textbf{76.4}           & \textbf{81.2}    &  \textbf{82.2}    & \textbf{77.1}     & \textbf{83.5}           & \textbf{87.4}    &  \textbf{87.7}   \\ 
\bottomrule
\end{tabular}%
}
\vspace{-1.5mm}
\caption{Performance (\%) comparison with GZSL methods on \textit{CUB, AWA1}, and \textit{AWA2}.} \label{tab:three}
\vspace{1.5mm}
\end{table*}

% . Average accuracy (\%) with 95\% confidence intervals.

% figure 5. effeteness of parameters and confusion matrix
\begin{figure}[t]
	\centering
    \subfigure[\label{fig:conf:a}w/o normalization]{\hspace{0mm}\includegraphics[width=0.35\columnwidth]{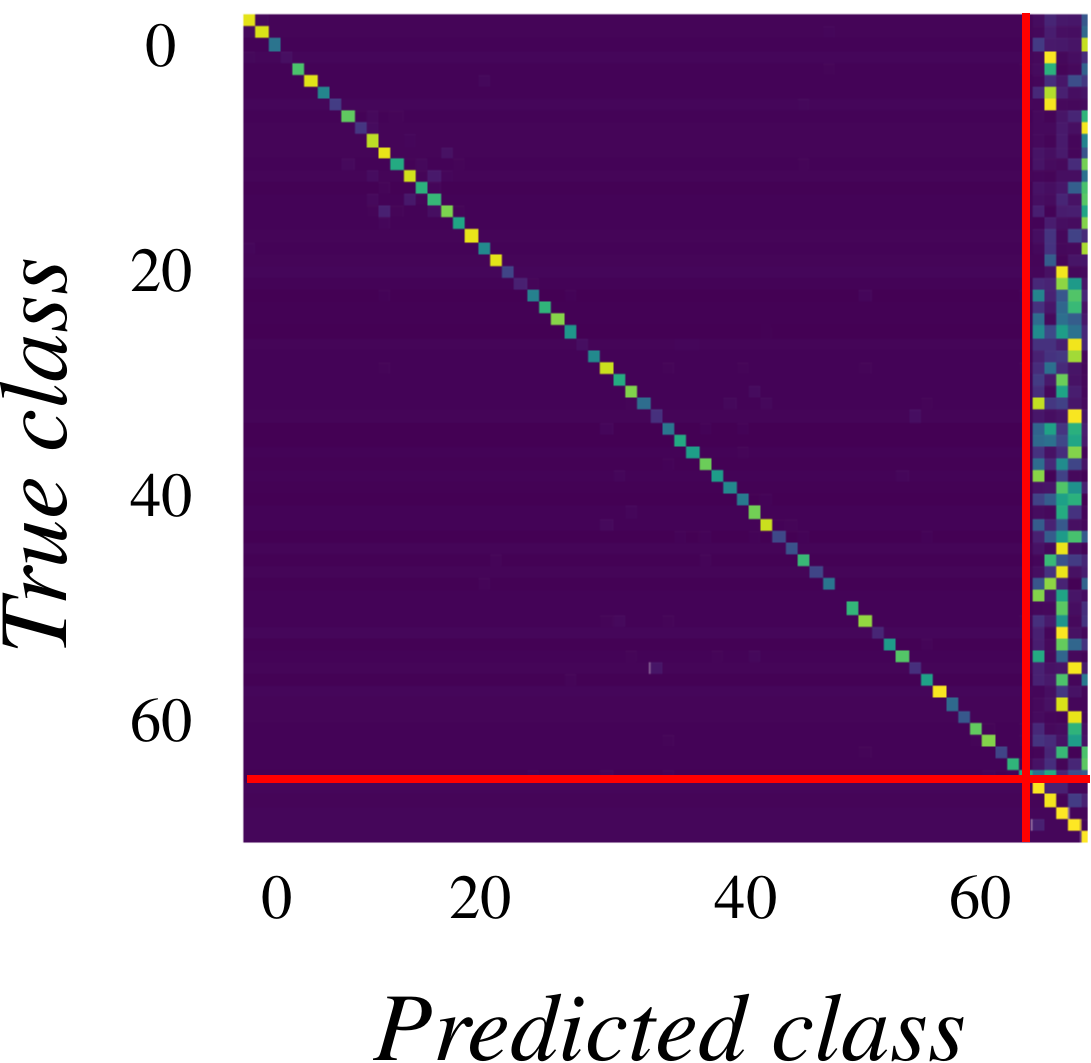}\hspace{4mm}}
    \subfigure[\label{fig:conf:b}w/ normalization (ours)]{\hspace{4mm}\includegraphics[width=0.35\columnwidth]{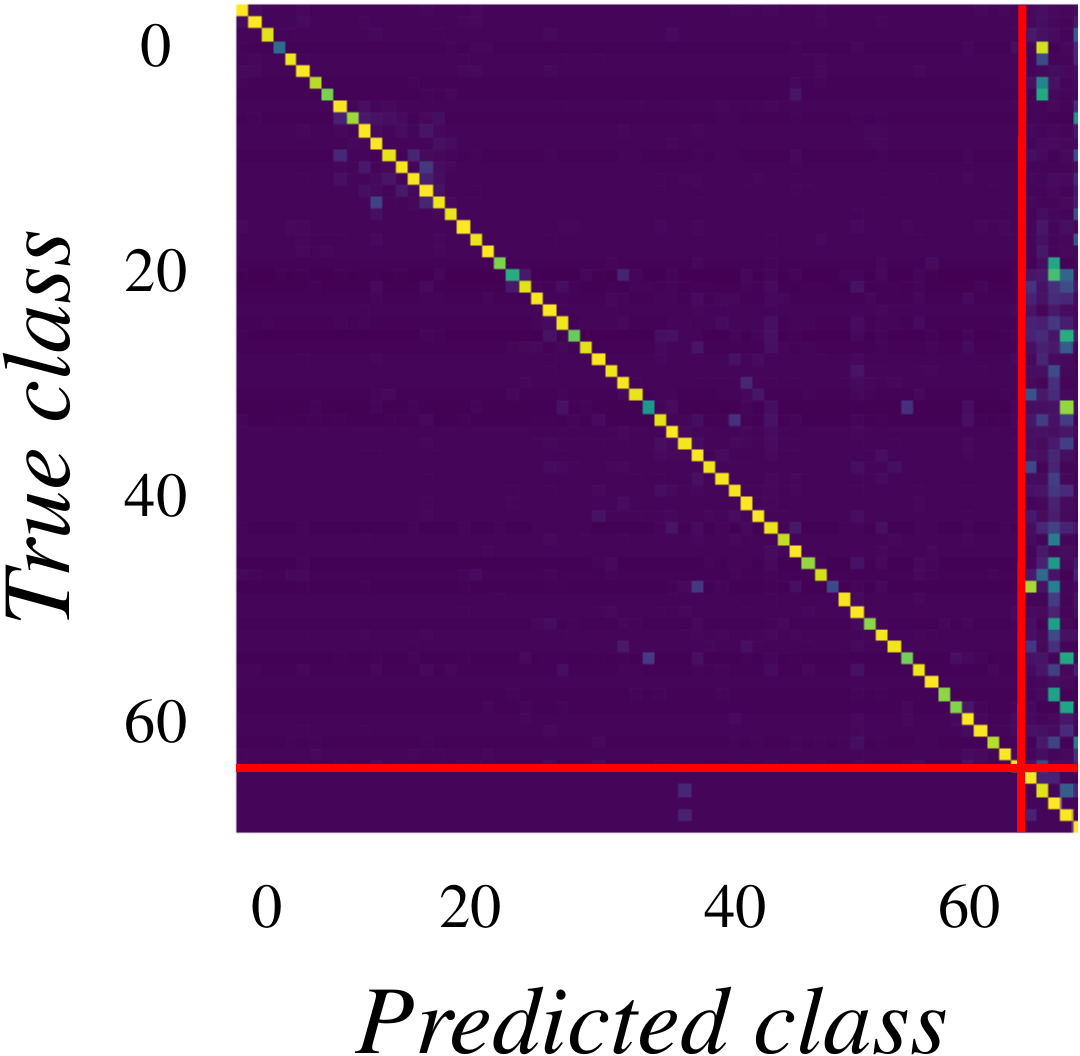}\hspace{4mm}}
    \caption{Confusion matrices on \textit{mini-ImageNet}, where the rightmost ones after red lines are novel classes.}
	\label{fig:confusion}%Confusion matrices
	\vspace{3mm}
% \end{figure}

% \begin{figure}[t]
	\centering
	\includegraphics[height=1.8mm]{./figure/distribution_key.pdf}
    \subfigure[\label{fig:conf:a}w/o normalization]{\hspace{0mm}\includegraphics[width=0.4\columnwidth]{./figure/figure_dist_5shot.pdf}\hspace{1mm}}
    \subfigure[\label{fig:conf:b}w/ normalization (ours)]{\hspace{1mm}\includegraphics[width=0.4\columnwidth]{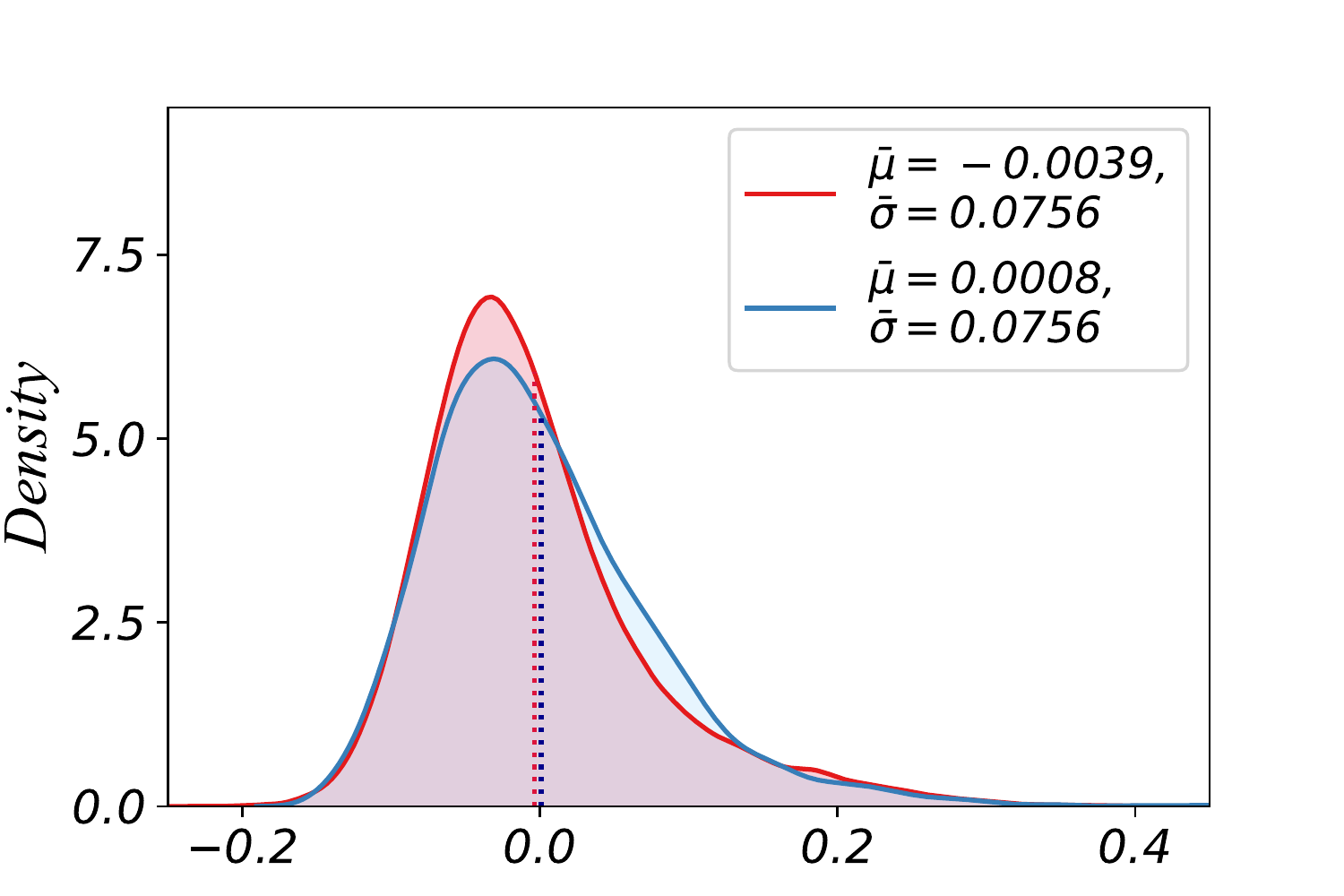}\hspace{0mm}}
    \caption{Weight distributions of base and novel classifiers of ResNet-18 on \textit{tiered-ImageNet} for 5-shot.}
	\label{fig:dist:ours}%The weight distributions of base and novel
	\vspace{-3mm}
\end{figure}

% overall -zero-base GFSL beats GFSL in large margin.
% more data, more novel performance, gap larger.
\smalltitle{Overall performance}
Table \ref{tab:tiered} summarizes the overall performance of the compared GFSL methods on \textit{ImageNet}. Even without base data, it is clearly observed that our MVCN method outperforms the other GFSL methods exploiting base samples. In \textit{tiered-ImageNet}, which is a more challenging scenario where base and novel classes are quite different, MVCN seems to be even more effective when it is in \textit{mini-ImageNet}. This is probably because a larger number of base classes can be pretrained in \textit{tiered-ImageNet} than in \textit{mini-ImageNet}, and the simple transfer learning scheme of MVCN takes more advantage of this well-trained knowledge of many base classes. Especially when we have more novel samples like 5 or 10-shot, the performance gap between ours and the other methods becomes larger as shown in Table \ref{tab:tiered}. The underlying reason is mean shifting phenomenon gets stronger as we fine-tune more novel classes (see Figure \ref{fig:shifting}). Although XtarNet \cite{YoonKSM20} occasionally performs slightly better on base classes than ours, MVCN still manages to outperform XtraNet on novel classes with clear margins. Furthermore, Table \ref{tab:three} shows that MVCN is superior to all the SOTA GZSL methods on \textit{CUB}, \textit{AWA1}, and \textit{AWA2}. Note that each image in these datasets is augmented with textual attributes that are crucially utilized by most GZSL methods while our method does not exploit any extra information.

% Also, note that some existing methods including XtraNet need to learn extra modules, and they would get much worse if a pretrained model is solely used just like our CN method. 

%  Through the fact that our method using a pre-trained method acquires existing performance just by learning a limited classifier, we confirm that the existing method eventually learns a base and novel relationship. 

\smalltitle{Qualitative results}
The confusion matrices of with and without normalization are presented in Figure \ref{fig:confusion}. We can see that the linear classifier without normalization tends to erroneously classify base samples as novel classes, but not vice versa. This observation confirms our claim that decision boundaries of novel classes get unnecessarily large in the fine-tuned classifier. However, after applying our normalization scheme, it is well observed that the classifier seems no longer to be confused about base classes with novel classes.

% We visualize the feature vectors and decision boundaries in 2-dimension space in \ref{fig:decision}. We randomly choose four classes from \textit{tiered-ImageNet} as base classes and we add one class as the novel class. We use the same ResNet-18 as the \ref{tab:tiered}, but we add an additional linear projection module before the linear classifier to be 2-dimension. As can be seen, 
% \kim{Wokring on...}
In addition, as shown in Figure \ref{fig:dist:ours}, we can observe that MVCN actually resolves \textit{mean shifting phenomenon} and achieves a balanced variance between base and novel classes as the classifier with MVCN has a $\bar \mu_{novel}$ 62 times smaller than it is without normalization. Also, we show that $\bar \sigma_{base}$ and $\bar \sigma_{novel}$ become almost the same after normalization, implying that there is no bias toward either novel or base classes.

% table - Ablation 
\begin{table}[]
\renewcommand{\arraystretch}{2.0}
\centering
\resizebox{\columnwidth}{!}{%
\begin{tabular}{ccc|ccc|ccc}
\toprule 
\multicolumn{3}{c|}{{\textbf{Methods/ Shots}}} & \multicolumn{3}{c|}{{\textbf{1-shot}}} & \multicolumn{3}{c}{{\textbf{5-shot}}} \\ 
\cmidrule(r){1-9}
\textbf{MC} & \textbf{VB} & \textbf{LO} & \textbf{Novel}        & \multicolumn{1}{c||}{\textbf{Base}}         & \textbf{All}         & \textbf{Novel}         & \multicolumn{1}{c||}{\textbf{Base}}        & \textbf{All}    \\
\cmidrule(r){1-9} 
 \xmark & \xmark & \xmark & 60.34 $\pm$ 0.62  & \multicolumn{1}{c||}{47.28 $\pm$ 0.07}  & 53.81 $\pm$ 0.31 &79.97 $\pm$ 0.44   & \multicolumn{1}{c||}{38.44 $\pm$ 0.07} & 59.21 $\pm$ 0.22   \\
 \cmidrule(r){1-9}
\cmark     & \xmark & \xmark & 36.19 $\pm$ 0.83           & \multicolumn{1}{c||}{\textbf{75.11}  $\pm$ \textbf{0.07}} & 55.65  $\pm$ 0.39         & 73.13 $\pm$ 0.48            & \multicolumn{1}{c||}{\textbf{68.91 $\pm$ 0.10}}          & 71.02 $\pm$ 0.28       \\
\cmark        & \cmark     & \xmark & 40.86  $\pm$ 0.66           & \multicolumn{1}{c||}{73.95   $\pm$ 0.12}         & 57.41    $\pm$ 0.30       & 73.49   $\pm$ 0.59 & \multicolumn{1}{c||}{68.71    $\pm$ 0.10}       & 71.10   $\pm$ 0.29          \\
\cmark        & \cmark            & \cmark           & \textbf{51.72}  $\pm$ \textbf{0.64} & \multicolumn{1}{c||}{65.81 $\pm$ 0.08}            & \textbf{58.77} $\pm$ \textbf{0.35}  & \textbf{75.20} $\pm$ \textbf{0.46}              & \multicolumn{1}{c||}{67.62 $\pm$ 0.09}  & \textbf{71.22} $\pm$ \textbf{0.26}   \\
\bottomrule
\end{tabular}
}
\vspace{-1.5mm}
\caption{Ablation study on \textit{mini-ImageNet} to analyze the effectiveness of three components of our method, which are mean centering (MC), variance balancing (VB), and linear optimization (LO).} \label{tab:abalation}
% \vspace{3.5mm}
\end{table}

% add in AAAI

\smalltitle{Ablation study}
To examine the effect of each component of our method, which are online mean centering (MC), offline variance balancing (VB), and post linear optimization (LO), we conduct an ablation study in zero-base GFSL using \textit{mini-ImageNet}. Table \ref{tab:abalation} shows that applying mean centering solely improves the accuracy of base classes by more than $25 \%$ in all shot cases, which shows its effectiveness of mitigating the novel-bias problem. In addition to MC, variance balancing tends to make the model more accurate for novel classes yet a bit less accurate for base classes. Considering the results of Table \ref{tab:tiered}, note that MC together with VB beats the SOTA GFSL method XtarNet by large margins in 5-shot case, and shows comparable performance to the best 1-shot learning method LCwoF. Finally, when optimizing linear parameters with a tiny overhead, we can further improve the performance with $1.46 \%$ higher accuracy for the 1-shot case. Through all the components, we get $4.96 \%$ and $12.01 \%$ performance gains for 1-shot and 5-shot, respectively.

%(a) Influence of the proposed learnable parameter $\gamma$ and $\beta$ on the training after freezing the entire others weights of the model, results on \textit{Mini-ImageNet}.
% Figure5-effect of parameter. 
% As presented in Figure \ref{fig:effect} We evaluate the effect of the learnable parameter $\gamma$ and $\beta$. We can improve the performance of novel classes, which is . % experiment
\section{Conclusion} \label{sec:con}
In this paper, we conducted the first study on zero-base GFSL, where we need to fine-tune a joint classifier with only a few samples of novel classes. Through a systematic analysis, we discovered that mean shifting phenomenon was the critical reason behind a novel-biased classifier, but the existing GFSL methods have been trying to equalize only the variance of weights. Based on our findings, we proposed a simple yet effective weight normalization method without using any base samples, which can even beat the existing GFSL methods that utilize the base dataset. Even though this work dealt with only linear classifiers, which is our limitation, our next plan is to extend our analysis to cover various types of classifiers.
 % conclusion

\section{Acknowledgments}
This work was supported in part by Institute of Information \& communications Technology Planning \& Evaluation (IITP) grants funded by the Korea government(MSIT) (No.2022-0-00448, Deep Total Recall: Continual Learning for Human-Like Recall of Artificial Neural Networks, No.RS-2022-00155915, Artificial Intelligence Convergence Innovation Human Resources Development (Inha University)), and in part by the National Research Foundation of Korea (NRF) grants funded by the Korea government (MSIT) (No.2021R1F1A1060160, No.2022R1A4A3029480).

% Use \bibliography{yourbibfile} instead or the References section will not appear in your paper
\bibliography{aaai23}   % name your BibTeX data base

\newpage
\onecolumn
\setcounter{table}{0}
\setcounter{figure}{0}
\renewcommand{\thetable}{A\arabic{table}}  
\renewcommand{\thefigure}{A\arabic{figure}}
\appendix

\section*{Appendix of Better Generalized Few-Shot Learning Even Without Base Data} \label{sec:appendix}

In this appendix, we describe more details about our experimental results including all the implementation details and the values of hyperparameters. Also, we present more experimental results, which are of independent interests but cannot be contained in the main manuscript due the space limit, namely (1) a more detailed ablation study, (2) comparison with the other baselines, (3) a performance test of applying mean centering either to all classifiers or to novel classifiers, (4) the effectiveness of our normalization method when exploiting base data, (5) considering more feature extractors, and (6) the effect of linear parameters.

\section{Implementation Details}
We apply standard data augmentation, with randomly resized crops, color jittering and horizontal flips.
In the pre-training stage, we train each model on the full dataset of base samples for about 100 or less epochs, and then we fine-tune only the classifier for novel classes with a much smaller learning rate during hundreds of iterations while freezing the feature extractor and batch normalization parameters. In all the models and datasets, we follow the settings widely adopted by the existing works, each detail of which is as follows.

\smalltitle{Training ResNet-12 on Mini-ImageNet} We train ResNet-12 on base classes for 90 epochs with SGD optimizer with momentum 0.9 and the initial learning rate 0.1 that is decayed by 0.1 at 60 epochs. We use a weight decay of 5e-4, and batch size is 60. For the fine-tuning stage, we train the classifier on novel classes for 500 iterations with learning rate 0.005 for 1-shot, 0.003 for 5-shot, and 0.001 for 10-shot learning. For class-wise learnable parameters (i.e., $\gamma$ and $\beta$), we fine-tune the classifier again for 500 iterations with the same learning rate for 1-shot, 50 iterations for 5-shot, 5 iterations for 10-shot. Similar to \cite{KuklevaKS21}, we use \textit{Dropblock} \cite{GhiasiLL18} as a regularizer for backbones and and change the number of filters from (64, 128, 256, 512) to
(64, 160, 320, 640).

\smalltitle{Training ResNet-18 on Tiered-ImageNet} ResNet-18 is pre-trained on base classes for 120 epochs with SGD optimizer with momentum 0.9 and the initial learning rate 0.1 that is decayed by 0.1 at 60 and 90 epochs. The weight decay is set to 5e-4, and batch size is 60. The classifier is fine-tuned for 500 iterations with learning rate 0.003 for 1-shot and 0.001 for 5, 10-shot learning. When performing linear optimization, we learn $\gamma$ and $\beta$ by fine-tuning the classifier again with learning rate 0.01 for 500 iterations for 1-shot, 50 iterations for 5-shot, 5 iterations for 10-shot. As in \textit{Tiered-ImageNet}, we use \textit{Dropblock}.

\smalltitle{Training ResNet-50 on ImageNet-800} For base classes, we train ResNet-50 for 90 epochs with SGD optimizer with momentum 0.9 and the initial learning rate  0.1 that is decayed by 0.1 at 30 and 50 epochs. The weight decay and batch size are set to 1e-4 and 256, respectively. Fine-tuning the classifier for novel classes is performed for 100 iterations with learning rate 0.003 for 1-shot and 0.001 for 5, 10-shot learning. Class-wise parameters are obtained by another fine-tuning stage of the classifier for 500 iterations with learning rate 0.05.

\smalltitle{Details of the datasets CUB, AWA1 and AWA2} \textit{CUB}, \textit{AWA1}, and \textit{AWA2} are the most widely used benchmark datasets for zero-shot based GFSL methods (a.k.a. GZSL methods). The \textit{CUB-200-2011} (\textit{CUB}) dataset \cite{welinder2010caltech} is a fine-grained dataset containing 11,788 images of 200 different classes of birds, which are then split into 150 training classes, and 50 testing classes. All images in \textit{CUB} are annotated with 15 part locations, 1 subcategory labels, 312 attributes and bounding boxes. 
\textit{Animal with Attribute 1} (\textit{AWA1}) \cite{LampertNH09} contains 30,475 images over 50 classes, which are then split into 40 classes for training and 10 classes for test where each category is annotated with 85 semantic attributes. \textit{AWA2} \cite{XianLSA19} has 37,322 images and uses the same 50 classes as \textit{AWA1}, with 85 semantic attributes. We adopt the split settings introduced in \cite{XianLSA19}.

\smalltitle{Training ResNet-101 on CUB, AWA1 and AWA2}
For the pre-training stage, we commonly retrain ResNet-101, which is pretrained for ImageNet, for base classes. For CUB, we train the model for 60 epochs with the initial learning rate 0.001 decayed by 20 and 40 epochs, and batch size is 256. For \textit{AWA1}, in particular, we should train only the base classifier for 100 epochs because \textit{AWA1} only provides feature vectors, not raw base samples. The learning rate starts from 0.1, and then is decayed at 50 and 80 epochs by 0.1 in \textit{AWA1}. Also, the batch size is set to 256 as in \textit{CUB}. The base model is trained on \textit{AWA2} for 60 epochs with the initial learning rate 0.1 that is decayed by 0.1 at 20 and 40 epochs and the batch size is 60 for \textit{AWA2}. For all three datasets, we use SGD optimizer with momentum 0.1 and a weight decay of 5e-4. In the fine-tuning stage of \textit{CUB}, the classifier is fine-tuned for novel classes during 1000 iterations for 1, 2-shot and 800 iterations for 5, 10-shot with learning rate 0.0002. In \textit{AWA1}, fine-tuning is done for 1000 iterations for 1, 2-shot and 500 iterations for 5, 10-shot with learning rate 0.0005. In \textit{AWA2}, we do the same process as in \textit{AWA1} but with the learning rate 0.0001. We do not learn class-wise parameters $\gamma$ and $\beta$ in these experiments.

\smalltitle{Measurement details} 
During the test time, we evaluate our method by taking an average over 600 randomly sampled tasks. We consider 64+5-way classification on \textit{Mini-ImageNet}, and 200+5-way classification on \textit{Tiered-ImageNet} with 1, 5, and 10 novel train samples (i.e. shots). The 10-shot task has never been tested in the literature of GFSL, and hence we report all the results for 10-shot tasks obtained on our own by executing the implementations of the compared methods except we re-implement LCwoF \cite{KuklevaKS21}. For XtarNet \cite{YoonKSM20}, we reproduce all the results by the authors' implementation, and the rest of the results are referred from \cite{KuklevaKS21}.

As shown in the introduction section, we evaluate our method using \textit{ImageNet-800}, which divides \textit{ImageNet} 1K classes into 800 base classes and novel 200 classes, following the settings of \cite{Chen00D021} about standard FSL. The image size and volume of data are the same as those of the original ImageNet dataset. Note that this large dataset has never been tested by any of the existing GFSL methods.

\section{Additional Experiments}

\subsection{Ablation Study}
Table \ref{tab:im800} presents the detailed results of another ablation study using \textit{ImageNet-800}, whose summary is shown in Figure \ref{fig:imagenet}. As mentioned in the introduction section, our normalization method almost reaches to the upper bounds of base and novel classes particularly when we are able to use a powerful pretrained model such as ResNet-50 trained from abundant data \textit{ImageNet-800}.
The upper bound of novel classes is referred by a previous work \cite{Chen00D021} of measuring only the novel classes, and the upper bound of base classes is taken by the original performance of the pretrained model.

\begin{table*}
% table - imagenet-800 acc
\renewcommand{\arraystretch}{2.0}
\centering
\resizebox{\textwidth}{!}{%
\begin{tabular}{lccc|ccc|ccc|ccc}
\toprule 
\multicolumn{4}{c|}{\textbf{Methods/ Shots}} & \multicolumn{3}{c|}{\textbf{1-shot}}   & \multicolumn{3}{c|}{\textbf{5-shot}}   & \multicolumn{3}{c}{\textbf{10-shot}}                               \\ \cmidrule(r){5-13}
\textbf{\begin{tabular}[c]{@{}l@{}}Method \end{tabular}} &
\textbf{\begin{tabular}[c]{@{}l@{}}MC\end{tabular}} & 
\textbf{\begin{tabular}[c]{@{}l@{}}VB\end{tabular}} & \textbf{\begin{tabular}[c]{@{}l@{}}LO\end{tabular}} &
 \textbf{Novel} & \multicolumn{1}{c||}{\textbf{Base}} & \textbf{All} & \textbf{Novel} & \multicolumn{1}{c||}{\textbf{Base}} & \textbf{All} & \textbf{Novel} & \multicolumn{1}{c||}{\textbf{Base}} & \textbf{All} \\ \cmidrule(r){1-13}

Novel upper-bound \cite{Chen00D021} & \xmark & \xmark & \xmark &
\textbf{89.70 $\pm$ 0.19}    & \multicolumn{1}{c||}{\_}            & \_           & \textbf{96.14 $\pm$ 0.08}    & \multicolumn{1}{c||}{\_}            & \_           & \textbf{97.58 $\pm$ 0.25}             & \multicolumn{1}{c||}{\_}            & \_           \\
Base upper-bound  & \xmark & \xmark & \xmark  &
\_             & \multicolumn{1}{c||}{\textbf{76.01}}         & \_           & \_             & \multicolumn{1}{c||}{\textbf{76.01}}         & \_           & \_             & \multicolumn{1}{c||}{\textbf{76.01}}         & \_           \\
w/o normalization  & \xmark & \xmark & \xmark &
82.57 $\pm$ 0.83    & \multicolumn{1}{c||}{51.60 $\pm$ 0.1}    & 67.08 $\pm$ 0.39  & 95.86 $\pm$ 0.32    & \multicolumn{1}{c||}{45.10 $\pm$ 0.08}   & 70.48 $\pm$ 0.18  & 97.09 $\pm$ 0.24    & \multicolumn{1}{c||}{38.53 $\pm$ 0.08}   & 67.81 $\pm$ 0.15  \\ \cmidrule(r){1-13}
MC & \cmark & \xmark & \xmark &                                    22.70 $\pm$ 0.92    & \multicolumn{1}{c||}{75.13 $\pm$ 0.46}   & 48.91 $\pm$ 0.46  & 68.25 $\pm$ 0.9     & \multicolumn{1}{c||}{74.70 $\pm$ 0.07}   & 71.48 $\pm$ 0.45  & 80.35 $\pm$ 0.62    & \multicolumn{1}{c||}{74.41 $\pm$ 0.07}   & 77.38 $\pm$ 0.31  \\
MC + VB & \cmark & \cmark & \xmark  &
 49.56 $\pm$ 0.11    & \multicolumn{1}{c||}{71.97 $\pm$ 0.57}   & 60.77 $\pm$ 0.05  & 82.13 $\pm$ 0.74    & \multicolumn{1}{c||}{73.34 $\pm$ 0.07}   & 77.74 $\pm$ 0.36  & 89.05 $\pm$ 0.51    & \multicolumn{1}{c||}{73.18 $\pm$ 0.06}   & 81.11 $\pm$ 0.25  \\
MC + VB + LO & \cmark & \cmark & \cmark  &
 76.32 $\pm$ 0.96    & \multicolumn{1}{c||}{67.78 $\pm$ 0.09}   & \textbf{72.05 $\pm$ 0.47}  & 92.07 $\pm$ 0.45    & \multicolumn{1}{c||}{70.98 $\pm$ 0.07}   & \textbf{81.52 $\pm$ 0.23} & 95.17 $\pm$ 0.31    & \multicolumn{1}{c||}{70.35 $\pm$ 0.07}   & \textbf{82.76 $\pm$ 0.17} \\
\bottomrule
\end{tabular}
}
\vspace{-1.5mm}
\caption{Ablation study on \textit{ImageNet-800} to analyze the effectiveness of different components in our method. MC is the mean centering, VB is the variance balancing, and LO is the re-scaling with additional linear parameters.} \label{tab:im800}
\vspace{1.5mm}
\end{table*}

\begin{table*}[]
\renewcommand{\arraystretch}{2.4}
\centering
\resizebox{\textwidth}{!}{%
\begin{tabular}{l|ccc|ccc|ccc}
\toprule
\multirow{2}{*}{\textbf{Methods/ Shots}} & \multicolumn{3}{c|}{\textbf{1-shot}} & \multicolumn{3}{c|}{\textbf{5-shot}} & \multicolumn{3}{c}{\textbf{10-shot}} \\ \cline{2-10} 
                   & \textbf{Novel} & \multicolumn{1}{c||}{\textbf{Base}} & \textbf{All} & \textbf{Novel} & \multicolumn{1}{c||}{\textbf{Base}} & \textbf{All} & \textbf{Novel} & \multicolumn{1}{c||}{\textbf{Base}} & \textbf{All} \\ \hline
Linear classifier         & 60.34 $\pm$ 0.62    & \multicolumn{1}{c||}{47.28 $\pm$ 0.07}   & 53.81 $\pm$ 0.31  & 79.97 $\pm$ 0.44    & \multicolumn{1}{c||}{38.44 $\pm$ 0.07}   & 59.21 $\pm$ 0.22  & 84.06 $\pm$ 0.39    & \multicolumn{1}{c||}{34.59 $\pm$ 0.07}   & 59.33 $\pm$ 0.19  \\
Freezing base classifier             & 60.53 $\pm$ 0.62    & \multicolumn{1}{c||}{46.40 $\pm$ 0.06 }  & 53.46 $\pm$ 0.30   & 80.00 $\pm$ 0.44    & \multicolumn{1}{c||}{36.66 $\pm$ 0.06}  & 58.33 $\pm$ 0.23  & 84.20 $\pm$ 0.38    & \multicolumn{1}{c||}{33.13 $\pm$ 0.08}   & 58.67 $\pm$ 0.20   \\
Cosine   classifier        & \textbf{61.96 $\pm$ 0.79}    & \multicolumn{1}{c||}{37.59 $\pm$ 0.09}   & 49.78 $\pm$ 0.37  & \textbf{80.21 $\pm$ 0.56}    & \multicolumn{1}{c||}{30.37 $\pm$ 0.08}   & 55.29 $\pm$ 0.25  & \textbf{84.46 $\pm$ 0.49}    & \multicolumn{1}{c||}{28.06 $\pm$ 0.08}   & 56.26 $\pm$ 0.22  \\
L1 regularization & 39.29 $\pm$ 0.73    & \multicolumn{1}{c||}{55.87 $\pm$ 0.14}   & 47.58 $\pm$ 0.35  & 27.69 $\pm$ 0.57    & \multicolumn{1}{c||}{67.78 $\pm$ 0.09}   & 47.74 $\pm$ 0.25  & 31.63 $\pm$ 0.61    & \multicolumn{1}{c||}{65.34 $\pm$ 0.10}   & 48.49 $\pm$ 0.29  \\
L2 regularization & 43.68 $\pm$ 0.64    & \multicolumn{1}{c||}{69.72 $\pm$ 0.10}    & 56.70 $\pm$ 0.30   & 77.41 $\pm$ 0.45    & \multicolumn{1}{c||}{56.38 $\pm$ 0.10}   & 66.90 $\pm$ 0.24  & 82.71 $\pm$ 0.38    & \multicolumn{1}{c||}{53.70 $\pm$ 0.10}    & 68.21 $\pm$ 0.21  \\
VB in training   & 28.90 $\pm$ 0.77    & \multicolumn{1}{c||}{\textbf{77.12 $\pm$ 0.07}}   & 53.01 $\pm$ 0.37  & 63.55 $\pm$ 0.67    & \multicolumn{1}{c||}{\textbf{73.14 $\pm$ 0.07}}   & 68.34 $\pm$ 0.32  & 68.25 $\pm$ 0.59    & \multicolumn{1}{c||}{\textbf{74.71 $\pm$ 0.06}}   & 71.48 $\pm$ 0.28  \\
\cmidrule(r){1-10} 
MVCN (Ours)          & 51.72 $\pm$ 0.64   & \multicolumn{1}{c||}{65.81 $\pm$ 0.08}    & \textbf{58.77 $\pm$ 0.35}   & 75.20 $\pm$ 0.46      & \multicolumn{1}{c||}{67.62 $\pm$ 0.09}    & \textbf{71.22 $\pm$ 0.26}   & 78.7 $\pm$ 0.51     & \multicolumn{1}{c||}{68.06 $\pm$ 0.09}   & \textbf{73.38 $\pm$ 0.25}  \\ 
\bottomrule
\end{tabular}
}
\vspace{-1.5mm}
\caption{Performance comparison of various \naive approaches using \textit{mini-ImageNet}.} \label{tab:additional}
\vspace{1.5mm}
\end{table*}

% sample 5class on 15장 and query?

%\subsection{Hyperparameters}

\subsection{Comparison with Na\"ive Approaches}
In this section, we compare our method with a variety of na\"ive approaches, which are not necessarily introduced by the existing works. Specifically, we consider the following six methods: 
\begin{itemize}
    \item \textbf{Linear classifier} Fine-tuning only the default linear classifier
    \item \textbf{Freezing base classifier} Fine-tuning only the novel classifier while freezing the base classifier
    \item \textbf{Cosine classifier} Fine-tuning the cosine classifier, instead of a linear classifier, often used by the existing GFSL methods \cite{GidarisK18,QiBL18}
    \item \textbf{L1 regularization} L1 regularization for mitigating mean shifting phenomenon, instead of online mean centering
    \item \textbf{L2 regularization} L2 regularization for mitigating mean shifting phenomenon, instead of online mean centering
    \item \textbf{VB in training} Variance balancing during fine-tuning while the others are the same as our normalization method
\end{itemize}
In Table \ref{tab:additional}, we show that our MVCN method takes always the best position among all the compared approaches. Note that even though we freeze the base classifier, the resulting accuracy of base classes does not remain the same, but abruptly decreases. This is because, as mentioned earlier, mean shifting phenomenon harms the overall balance between novel and base classifiers, which is not addressed by freezing the base classifier. Given the occurrence of mean shifting phenomenon, using the cosine classifier seems to make the situation even worse as observed that the accuracy of base classes gets lower than it is when using the linear classifier. This is due to the fact that the cosine classifier is effective to equalize the weight variances between novel classes and base classes by increasing the weight norms of the novel classifier but decreasing those of the base classifier. Consequently, the novel bias problem gets more 
severe when using the cosine classifier without addressing mean shifting phenomenon. Strong L1 or L2 regularization improves the accuracy of base classes, implying its ability to mitigate mean shifting phenomenon. However, this time the model gets significantly in accurate for novel classes, which is because L1 or L2 regularization tends to reduce the weight variance (i.e., weight norm) of novel classes. Thus, these regularization approaches cannot achieve a desirably balanced result between novel and base classes. Finally, we can see that variance balancing is much better to be performed as a post-training process in offline, rather than being performed with mean centering in the process of training.

% (1) We fine-tune using the default linear classifier. 2) We fine-tune only the novel classifier while freezing the base classifier. 3) We fine-tune using the cosine classifier, which replaces cosine similarity to the inner product. Existing GFSL methods often use cosine classifiers instead of liner \cite{QiBL18,GidarisK18}. 4) and 5) We give strong L1 and L2 regularization constrain on novel classifier, which $\lambda$ is 1e-2, the default is 5e-4. 6) We do variance balancing during fine-tuning, not post processing. 

\begin{table*}[]
\scriptsize
\renewcommand{\arraystretch}{0.8}
\centering
\begin{tabular}{c|c|ccc|ccc}
\toprule 
\multirow{2}{*}{\textbf{GFSL Settings}} &
  \multirow{2}{*}{\textbf{Mean Centering / Shots}} &
  \multicolumn{3}{c|}{\textbf{1-shot}} &
  \multicolumn{3}{c}{\textbf{5-shot}} \\ \cline{3-8} 
   &
   &
  \textbf{Novel} &
  \multicolumn{1}{c||}{\textbf{Base}} &
  \textbf{All} &
  \textbf{Novel} &
  \multicolumn{1}{c||}{\textbf{Base}} &
  \textbf{All} \\ \cline{1-8} 
\multirow{2}{*}{\begin{tabular}[c]{@{}c@{}}Without base data\end{tabular}} &
  Novel only &
  \textbf{51.72} &
  \multicolumn{1}{c||}{65.81} &
  58.77 &
  \textbf{75.20} &
  \multicolumn{1}{c||}{67.62} &
  71.22 \\
                      & Both base and novel   & 44.65 & \multicolumn{1}{c||}{71.04} & 57.85 & 66.26 & \multicolumn{1}{c||}{\textbf{73.63}} & 69.95 \\ \cline{1-8}
\multirow{2}{*}{With base data} & Novel only  & 48.70 & \multicolumn{1}{c||}{69.17} & \textbf{58.94} & 73.32 & \multicolumn{1}{c||}{69.86} & \textbf{71.59} \\
                      & Both base and novel    & 45.38 & \multicolumn{1}{c||}{\textbf{72.04}} & 58.71 & 68.25 & \multicolumn{1}{c||}{73.40}  & 70.83 \\
\bottomrule
\end{tabular}
\vspace{-1.5mm}
\caption{Comparison on the accuracy with mean centering for novel and all classifiers with and without base data.}  \label{tab:setting}
\vspace{1.5mm}
\end{table*}

\begin{table*}[]
\renewcommand{\arraystretch}{1.2}
\scriptsize
\centering
\begin{tabular}{c|ccc|ccc}
\toprule
\multirow{2}{*}{\textbf{Backbone / Shots}} & \multicolumn{3}{c|}{\textbf{1-shot}} & \multicolumn{3}{c}{\textbf{5-shot}} \\ \cline{2-7} 
          & \textbf{Novel} & \multicolumn{1}{c||}{\textbf{Base}} & \textbf{All} & \textbf{Novel} & \multicolumn{1}{c||}{\textbf{Base}} & \textbf{All} \\ \hline
ResNet-12 & 58.56 $\pm$ 0.73    & \multicolumn{1}{c||}{60.31 $\pm$ 0.19}   & 59.44 $\pm$ 0.30  & 76.08 $\pm$ 0.53    & \multicolumn{1}{c||}{63.24 $\pm$ 0.15}   & 69.66 $\pm$ 0.25  \\
ResNet-18 & \textbf{62.11  $\pm$  0.70}   & \multicolumn{1}{c||}{\textbf{61.23  $\pm$  0.22}}  & \textbf{61.67  $\pm$  0.31} & \textbf{79.59  $\pm$  0.55}   & \multicolumn{1}{c||}{\textbf{66.30  $\pm$  0.20}}   & \textbf{72.34  $\pm$  0.28} \\
\bottomrule
\end{tabular}
\vspace{-1.5mm}
\caption{Comparison of different feature extractors on \textit{tiered-ImageNet}.} \label{tab:backbone}
\end{table*}

\begin{figure*}[t]
	\centering
	\includegraphics[height=2.7mm]{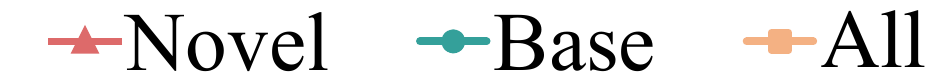} \\
    \subfigure[\label{fig:effect:1shot}1-shot]{\hspace{2mm}\includegraphics[width=0.25\columnwidth]{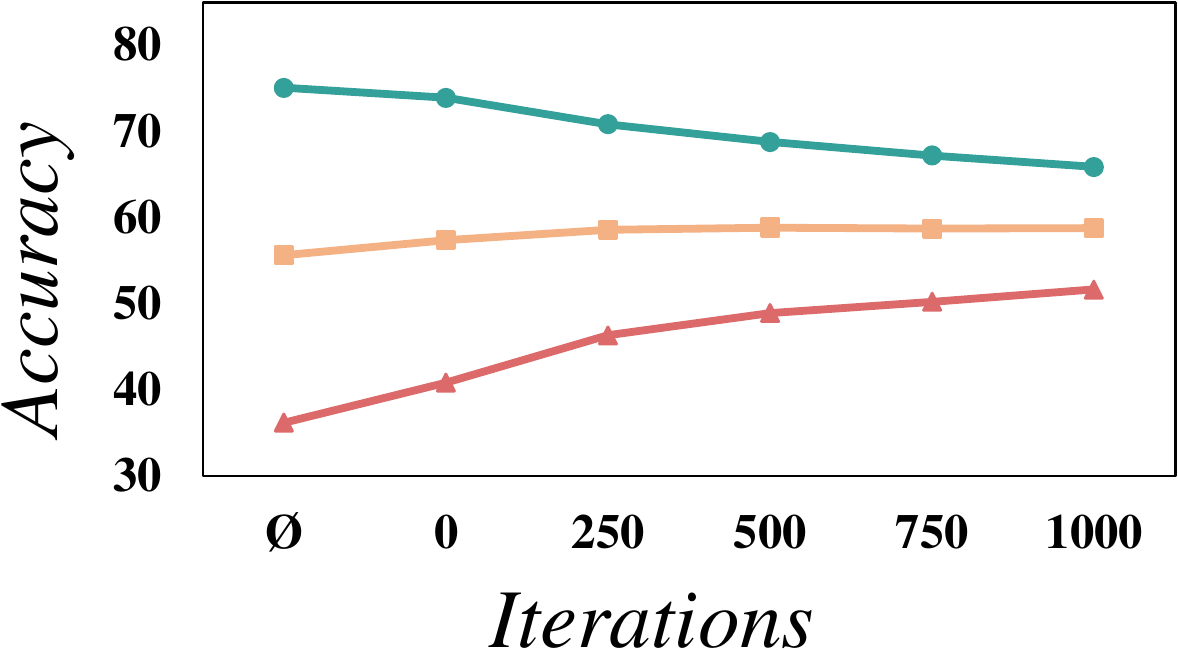}\hspace{2mm}} 
    \subfigure[\label{fig:effect:5shot}5-shot]{\hspace{2mm}\includegraphics[width=0.25\columnwidth]{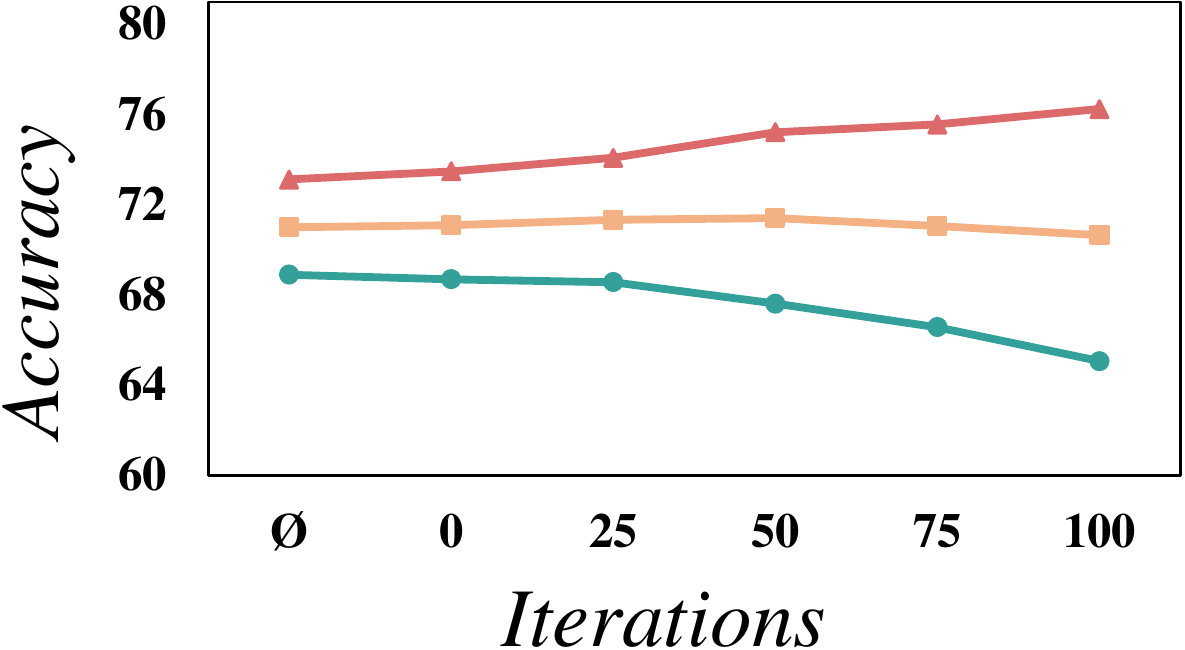}\hspace{2mm}}
    
    \caption{Evaluation on different iterations of linear parameters.}
	\label{fig:effect}
	\vspace{-2mm}
\end{figure*}

% Zero-base GFSL vs GFSL and all center vs novel center
\subsection{Mean Centering for Both vs. Novel Only}
We investigate how the experimental results can be changed if we apply mean centering not only to the novel classifier but also to the base classifier. As shown in Table \ref{tab:setting}, when conducting mean centering for both base and novel classifiers, the performance of base classes increases by more than $5 \%$ as \textit{negative} mean shifting (e.g., $\bar \sigma_{base} = -0.0029$ in Figure \ref{fig:shifting}) of base classes gets addressed as well. When it comes to the overall performance, however, it is still observed that mean centering only for the novel classifier is better than applying it to both classifiers. Thus, we can control the balance between base and novel classes by adjusting to what extent mean centering is performed to either the base or novel classifier.

\subsection{Performance Comparison With and Without Base Data}
Table \ref{tab:setting} also shows how much the final accuracy gets improved if we additionally use base data just like the normal GFSL scenario. Although there is a slight improvement in accuracy over GFSL (i.e., with base data) to zero-base GFSL, the performance gap is surprisingly subtle. We believe that our normalization method successfully plays the role of base data to such an extent that additional fine-tuning with a balanced dataset, which is obtained by undersampling, cannot make much difference.

% resnet12 vs resnet18 on Tiered-ImageNet
\subsection{Utilization of More Powerful Feature Extractors}
It has been reported that the better feature extractors (i.e., often deeper architectures) we use, the better the performance we can achieve in standard FSL \cite{ChenLKWH19, DhillonCRS20, TianWKTI20}. Unfortunately, this has not been the case in GFSL as the existing GFSL methods \cite{GidarisK18, ShiSSW20} do not seem to effectively leverage  of such a deeper architecture. In order to examine how well our CN method utilizes the power of feature extractors, we conduct the experiments using two different feature extractors, namely ResNet-12 and ResNet-18, both of which are trained on \textit{Tiered-ImageNet}. As shown in Table \ref{tab:backbone}, for 1-shot and 5-shot, CN with ResNet-18 leads to $2 \%$ and $2.7 \%$ improvements in accuracy, and $3.5 \%$ performance gain particularly for novel classes\footnote{Note that these results are better than those of LCwoF \cite{KuklevaKS21} in Table \ref{tab:tiered}.}. This indicates that our CN method works better when it is with a better (i.e., stronger) feature extractor, which has been as nicely observed in standard FSL.

% gamma beta training iterations
\subsection{Effect of Linear Parameters}
Finally, we analyze the influence of $\gamma$ and $\beta$ as shown in Figure \ref{fig:effect}. The $\emptyset$ iteration means the performance before variance balancing, and 0 iteration means the performance right after performing variance balancing. When it is hard for novel classes to be effectively learned compared to base classes, such as 1-shot, these linear parameters turn out to be effective to make a better balanced accuracy between base and novel classes as they increase the accuracy of novel classes. On the other hand, if the accuracy of novel classes is already high, such as 5-shot, it suffices to train linear parameters only for few or no iterations.

% Figure5-confusion
% we confirm  novel bias -> large room novel decision boundaries maybe large because confusion 
% \subsection{Confusion matrix With and Without normalization}
% The confusion matrices of with and without normalization are presented in Figure \ref{fig:confusion}. We can see that the linear classifier without normalization tends to erroneously classify base samples as novel classes, but not vice versa. This observation confirms our claim that decision boundaries of novel classes get unnecessarily large in the fine-tuned classifier. However, after applying our normalization scheme, it is well observed that the classifier seems no longer be confused about base classes with novel classes.

\end{document}